\pdfoutput=1
\documentclass[lettersize,journal,twoside]{IEEEtran}
\usepackage{amsmath,amsfonts}
\usepackage{amsthm}

\theoremstyle{definition}
\newtheorem{definition}{Definition}
\newtheorem{remark}{Remark}

\theoremstyle{plain}
\newtheorem{lemma}{Lemma}
\newtheorem{proposition}{Proposition}
\usepackage{algorithmic}
\usepackage{algorithm}
\usepackage{booktabs}
\usepackage{array}
\usepackage[caption=false,font=normalsize,labelfont=normalfont,textfont=normalfont]{subfig}
\usepackage{textcomp}
\usepackage{stfloats}
\usepackage{url}
\usepackage{verbatim}
\usepackage{graphicx}
\usepackage{cite}
\usepackage{microtype}
\usepackage{makecell}
\usepackage{multirow}
\usepackage{hyperref}
\usepackage{float}
\usepackage{caption}
\usepackage{enumitem}
\usepackage{eso-pic}
\AddToShipoutPictureBG*{\AtPageUpperLeft{%
  \put(\LenToUnit{0.5\paperwidth},\LenToUnit{-0.07in}){%
    \makebox[0pt][c]{\raisebox{-\height}{\parbox{0.9\paperwidth}{\centering\tiny
    \copyright~2026 IEEE. Personal use of this material is permitted. Permission from IEEE must be obtained for all other uses, in any current or future media, including reprinting/republishing this material for advertising or promotional purposes, creating new collective works, for resale or redistribution to servers or lists, or reuse of any copyrighted component of this work in other works. This work has been accepted for publication in IEEE Transactions on Robotics.}}}}}}

\def\IEEEbibitemsep{0pt plus .5pt}

\begin{document}

\title{RMRRT: Riemannian Barrier Metric RRT \\ for Inequality-Aware Steering on Equality Manifolds}
\author{Minhyeong Kang$^{1}$,~\IEEEmembership{Graduate Student Member,~IEEE}, and Sanghyun Kim$^{1,2}$,~\IEEEmembership{Member,~IEEE}%
\thanks{Received 25 July 2026; accepted 7 November 2026.
This research was supported in part by the National Research Foundation of Korea (NRF) grant funded by the Korea government (MSIT) under Grant RS-2024-00411007 and Grant RS-2024-00461583,
and in part by the National Research Council of Science \& Technology (NST) grant by the Korea government (MSIT) under Grant GTL25042-000.
This article was recommended for publication by Editor R.~Murrieta-Cid upon evaluation of the reviewers' comments.
\textit{(Corresponding author: Sanghyun Kim.)}}%
\thanks{$^{1}$Minhyeong Kang and Sanghyun Kim are with the Department of Mechanical Engineering, Kyung Hee University, Yongin 17104, Republic of Korea (e-mail: kingtyphoon@khu.ac.kr; kim87@khu.ac.kr).}%
\thanks{$^{2}$Sanghyun Kim is also with MetaMotionX, Inc., Seongnam, Republic of Korea.}}

\markboth{IEEE Transactions on Robotics. Preprint Version. Accepted November, 2026}%
{Kang and Kim: RMRRT: Riemannian Barrier Metric RRT for Inequality-Aware Steering on Equality Manifolds}

\maketitle

\begin{abstract}
This paper presents a motion planning framework that unifies equality and inequality constraints within a single geometric formulation for sampling-based planning in high-dimensional robotic systems. In conventional sampling-based planners, equality constraints are typically enforced through projection, whereas inequality constraints are handled separately through binary validity checks such as collision testing, often leading to inefficient exploration. To address this limitation, we propose Riemannian Barrier Metric RRT (RMRRT), which constructs a unified local geometry for planning on equality-constrained manifolds. RMRRT first builds an ambient barrier metric from inequality-sensitive barrier terms and then induces a tangent-space metric via a \(G\)-orthogonal projection associated with the equality constraints. The resulting tangent-space metric is used consistently in both steering and nearest-neighbor selection, biasing exploration away from nearby inequality boundaries while preserving first-order equality consistency. In this work, the metric is instantiated from signed-distance-based geometric proxy inequalities to provide collision-informative tangent-space directions; hard feasibility is enforced separately through standard validity checks. Experimental results show that RMRRT achieves a 100\% success rate across diverse constrained manipulation tasks in both simulation and real-world settings, while reducing planning time relative to representative constrained planning baselines. Ablation studies further demonstrate that the proposed metric improves exploration quality by reducing rejected samples and shortening path length. Experiment videos and source code are available at: \url{https://rcilab.khu.ac.kr/rmrrt/}

\end{abstract}


\section{Introduction} \label{section1}

\IEEEPARstart{S}{ampling}-based motion planning becomes particularly challenging when robotic systems are subject to simultaneous equality and inequality constraints \cite{mu2025efficient, stilman2007task, xian2017closed}. Equality constraints commonly arise from closed-chain kinematics \cite{bordalba2020randomized, jang2022motion, park2024constrained}, grasp maintenance \cite{liu2025planning, dogar2019multi, zimmermann2020multi}, and cooperative multi-arm manipulation \cite{shi2025distributed, alonso2015local}, restricting feasible configurations to an implicit manifold embedded in the ambient configuration space \cite{suh2011tangent, jaillet2012path, kim2016tangent}. Inequality constraints, such as obstacle avoidance \cite{quiroz2019whole, guo2012new}, self-collision \cite{grassmann2018smooth, Hu_2021_CVPR, cai2025just}, and joint limits \cite{stilman2007task, li2024quadratic}, further restrict admissible regions through state-dependent boundaries. Under such conditions, feasible configurations typically occupy only a small fraction of the ambient space, making exploration efficiency and steering quality central to planning performance. Fig.~\ref{fig_1} illustrates a representative coupled-constraint task in which three manipulators cooperatively transport a chair while maintaining closed-chain grasp constraints and avoiding collisions.

In conventional sampling-based planners, equality constraints are typically enforced through projection onto the constraint manifold, whereas inequality constraints are handled separately through binary validity checks such as collision testing \cite{suh2011tangent, jaillet2012path, kim2016tangent, berenson2009manipulation}. This decoupled treatment often leads to inefficient sample rejection and unstable steering, particularly in highly constrained environments where admissible regions are narrow or fragmented. Consequently, integrating equality and inequality constraints into a unified exploration framework remains a fundamental challenge in sampling-based motion planning. In this work, we address this challenge by proposing a unified geometric formulation that directly incorporates inequality-induced boundary geometry into exploration on equality-constrained manifolds.

\begin{figure}[!t]
\centering
\includegraphics[width=3.5in]{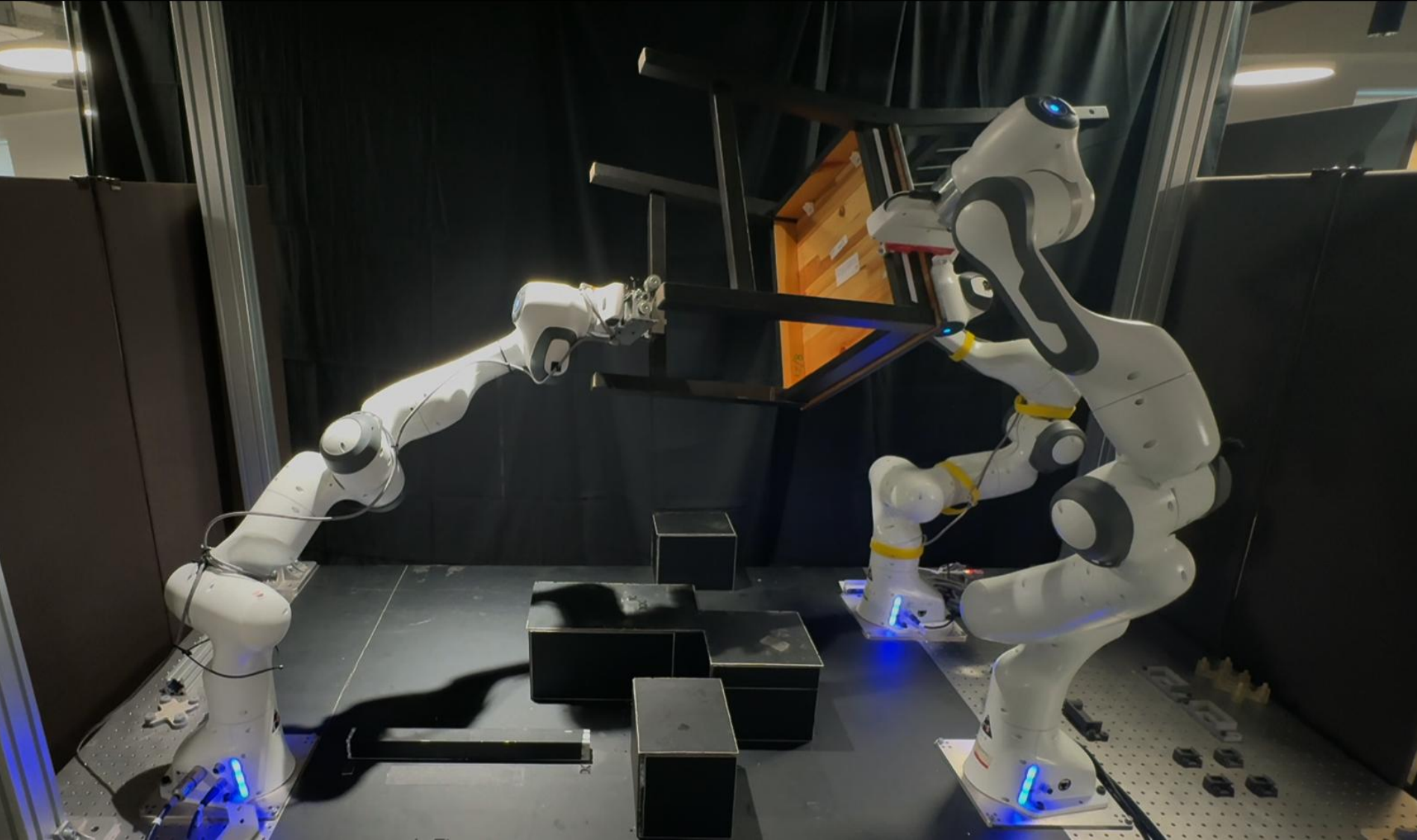}
\caption{
Real-world demonstration of RMRRT with three manipulators cooperatively transporting a chair under simultaneous equality and inequality constraints for closed-chain grasp maintenance and collision avoidance.}
\vspace{-0.1cm}
\label{fig_1}
\end{figure}

\vspace{-0.1cm}
\subsection{Related Work}

Sampling-based robot motion planning under constraints has been studied from multiple perspectives, largely distinguished by how equality and inequality constraints are represented and incorporated into planning. Prior work can be broadly organized into four categories:
(i) manifold-based approaches that plan directly on equality-constrained manifolds,
(ii) metric-based approaches that encode inequality-induced geometry into costs, distances, or local metrics,
(iii) data-driven approaches that accelerate exploration using learned or precomputed priors, and
(iv) trajectory-optimization approaches that treat equality and inequality constraints jointly over an entire trajectory.

\subsubsection{Manifold-Based Motion Planning}

Manifold-based motion planners operate directly on implicitly defined equality-constrained manifolds and aim to preserve equality feasibility throughout planning \cite{jang2022motion,suh2011tangent,jaillet2012path,kim2016tangent}. Atlas-based planners approximate the manifold using collections of local charts and transition between neighboring charts during exploration, improving robustness in highly curved regions \cite{jaillet2012path,jaillet2016path}. Tangent-space methods instead generate local motions in the Jacobian-defined tangent space and apply manifold-consistent correction, such as local projection, to obtain equality-feasible extensions \cite{suh2011tangent,kim2016tangent}. In most such approaches, however, tangent-space information is used primarily to maintain equality feasibility, whereas nearest-neighbor selection and steering are not consistently informed by nearby inequality boundaries. Consequently, inequality constraints often influence exploration only through post-hoc rejection by external validity checks, such as collision testing \cite{kingston2018sampling}. This separation can lead to repeated unproductive extension attempts in narrow or cluttered environments \cite{hsu2003bridge,pan2016fast}.

\subsubsection{Metric-Based Motion Planning}

Metric-based approaches incorporate inequality constraints into costs, distance measures, or state-dependent metrics to discourage motion toward constraint boundaries and bias exploration toward admissible regions. Classical examples include artificial potential fields and continuous obstacle-based trajectory costs \cite{khatib1986real,ratliff2009chomp}.
More recent Riemannian approaches construct non-Euclidean metrics from curvature, Hessian, or other differential information associated with inequality constraints \cite{manjunath2021safe,yang2019sampling}. Related reactive frameworks such as RMPflow use task-space Riemannian structures to combine multiple local motion objectives \cite{cheng2021rmpflow},
building on the Riemannian Motion Policies of Ratliff \emph{et~al.} \cite{ratliff2018riemannian}, which formalize a metric-as-geometry viewpoint for task-space motion generation. Although these approaches effectively encode inequality-aware local geometry, they are generally formulated in the ambient configuration space and do not directly generate equality-consistent motion on an implicitly defined constraint manifold. Additional projection or correction mechanisms are therefore required when they are applied to equality-constrained steering or tree expansion.

\subsubsection{Data-Driven Motion Planning}

Data-driven approaches accelerate constrained planning by exploiting learned or precomputed structure in the feasible space. Representative directions include learning low-dimensional latent representations of equality-constrained manifolds
\cite{park2024constrained,zhang2025local}, learning samplers or policies that bias expansion toward promising regions, and constructing offline roadmaps that reduce online search through precomputed feasible-state connectivity \cite{jang2022motion,csucan2012motion}. These methods can substantially improve planning efficiency when the learned prior or precomputed structure matches the target problem. Their performance, however, may depend on training coverage, prior quality, and the suitability of the offline representation under changes in environments, robot models, or constraints.

\subsubsection{Trajectory Optimization on Constraint Manifolds}

A complementary line of work addresses equality and inequality constraints jointly through trajectory optimization on constraint manifolds. Representative approaches include embedded sequential convex programming \cite{bonalli2019trajectory}, direct collocation for constrained robotic systems \cite{bordalba2022direct}, and analytic-IK-based bimanual planning \cite{cohn2024constrained}. These methods consider both classes of constraints at the trajectory level and typically refine an initialized trajectory toward a locally optimal solution. This setting differs from the satisficing, exploration-based problem considered here. Trajectory optimization generally improves a prescribed trajectory initialization according to an objective and a set of constraints, whereas sampling-based planning focuses on discovering a feasible connection through nonlocal exploration. The two paradigms therefore play complementary roles: trajectory optimization provides strong local refinement, while sampling-based planning can explore alternative connections and is less dependent on a favorable initialization.

\subsubsection{Structural Gap Across Existing Approaches}

Despite their respective strengths, the approaches reviewed above do not generally construct a single explicit local geometry that is simultaneously equality-consistent and inequality-aware throughout sampling-based exploration. In manifold-based planners, equality constraints define admissible first-order motion through tangent spaces and projection operators, but the geometry used for nearest-neighbor selection and steering remains largely insensitive to nearby inequality boundaries. In metric-based planners, inequality constraints reshape local distances or energies, but the resulting geometry is typically defined in the ambient space and does not directly yield equality-consistent motion. Data-driven methods can bias exploration toward favorable regions, but they generally do so through learned or precomputed priors rather than through an explicit geometric object coupling the two constraint classes. Trajectory optimization can treat both classes of constraints over an entire trajectory, but it typically operates by locally refining an initialization rather than by defining an exploration geometry for sampling-based search.

This separation reflects a geometric mismatch rather than merely a modeling choice.
A natural way to encode inequality constraints as a continuous local geometry is through a barrier-induced metric that grows near active constraint boundaries.
Such a metric is naturally defined in the ambient configuration space, whereas admissible first-order motions under equality constraints are restricted to a Jacobian-defined tangent space.
Simply combining an ambient barrier metric with a standard Euclidean null-space projector does not preserve compatibility between these two geometric structures.
The projection itself must be defined with respect to the metric that shapes local motion. This observation leads to the central construction of RMRRT: a metric-compatible tangent projection and its induced tangent-space metric, used consistently in nearest-neighbor selection, local steering, and manifold correction.
Together, these components provide a common geometric structure through which equality feasibility and inequality-aware exploration are coupled throughout the planner.

The work of Bohigas \emph{et~al.} \cite{bohigas2013planning} is particularly relevant because it also introduces a unified geometric construction for constrained sampling-based planning. Their method constructs a global, higher-dimensional geometric representation to obtain singularity-free paths in closed-chain systems. RMRRT instead constructs a local barrier-induced tangent-space metric that biases exploration away from obstacles and other active inequality boundaries. The two methods therefore address different geometric objectives and operate at different scales: Bohigas \emph{et~al.} use a global construction for singularity avoidance, whereas RMRRT uses a local metric for inequality-aware exploration on an equality-constrained manifold.

\vspace{-0.3cm}

\subsection{Overview of Our Approach}

As discussed above, existing sampling-based planners do not provide a single local geometry that is simultaneously equality-consistent and inequality-aware. This geometric mismatch requires a projection that is compatible with the metric defining the local geometry, which leads to the \(G\)-orthogonal projector and the induced tangent-space metric used in this work. Motivated by this observation, we propose Riemannian Barrier Metric RRT (RMRRT), a sampling-based planner that redefines local planning geometry on equality-constrained manifolds. The key idea is to endow equality-consistent motions with a barrier-informed Riemannian structure that captures the local influence of nearby geometric inequality boundaries. 

Concretely, we first construct a state-dependent ambient barrier metric from inequality-sensitive barrier terms, and then induce a metric on the tangent space through a barrier-compatible projection associated with the equality-constraint Jacobian. The resulting geometry not only preserves equality consistency, but also reshapes local exploration according to nearby inequality boundaries. In the evaluated instantiation, this boundary geometry is constructed from signed-distance-based geometric proxy inequalities, which capture the dominant robot--obstacle boundary geometry in cluttered constrained settings. Inequalities without tractable differential structure are handled separately through standard validity checks. Thus, the metric acts as an exploration bias, while hard feasibility is enforced independently by validity checking.

Based on this induced geometry, RMRRT performs local steering by selecting tangent directions that make first-order progress toward the target while penalizing motions toward nearby inequality boundaries. The same induced tangent-space metric is also used in nearest-neighbor selection, yielding a notion of local connectability that is consistent with equality-constrained expansion. As a result, inequality information affects the exploration mechanism itself rather than entering only through post-hoc validity checks.

\noindent The main contributions of this paper are summarized as follows:
\begin{itemize}
    \item We introduce an ambient barrier metric and a \(G\)-orthogonal projection that together induce a tangent-space metric for sampling-based planning on equality-constrained manifolds. This construction captures both manifold structure and inequality-boundary sensitivity within a single local geometry.

    \item Building on this construction, we develop a bidirectional RRT-based planner for inequality-aware exploration on equality-constrained manifolds. The proposed planner incorporates the induced tangent-space metric into the expansion process so as to bias tree growth away from nearby inequality boundaries while preserving equality consistency.

    \item We use the same induced tangent-space metric consistently in both steering and nearest-neighbor selection, establishing planner-wide geometric consistency during exploration. This aligns node selection with the local directions favored during expansion, in contrast to planners that rely on Euclidean proximity for selection and enforce equality consistency only during steering.

    \item We validate the proposed method through simulation and real-world experiments on constrained manipulation tasks involving multi-arm and humanoid systems. The results show improved planning efficiency and reliability over representative constrained planning baselines in challenging environments with coupled equality and inequality constraints~\cite{berenson2009manipulation,jaillet2012path,kim2016tangent,csucan2012motion,park2024constrained}.
\end{itemize}

\noindent The remainder of this paper is organized as follows. Section~\ref{section2} introduces the necessary preliminaries, including equality-constrained manifolds and barrier-based metrics. Section~\ref{section3} presents the proposed RMRRT methodology in detail. Section~\ref{section4} reports experimental results on a variety of constrained manipulation scenarios. Finally, Section~\ref{section5} concludes the paper and discusses directions for future work.

\section{Preliminaries} \label{section2}
\begin{figure*}[!t]
\centering
\subfloat[]{\includegraphics[width=3.5in]{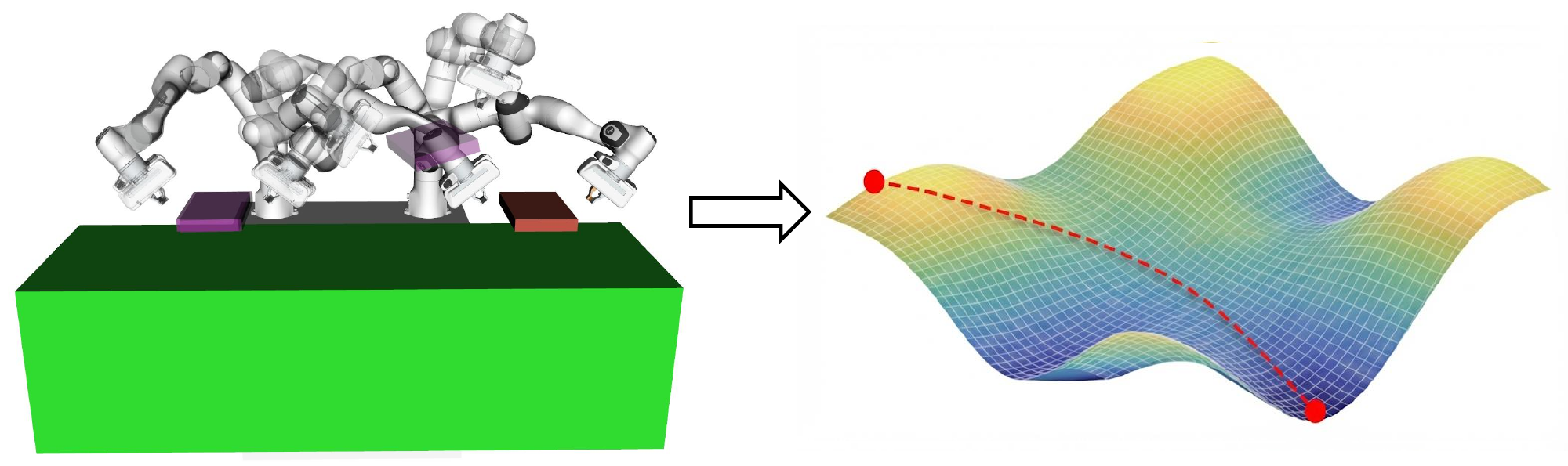}%
\label{fig2_a}}
\hfil
\subfloat[]{\includegraphics[width=3.5in]{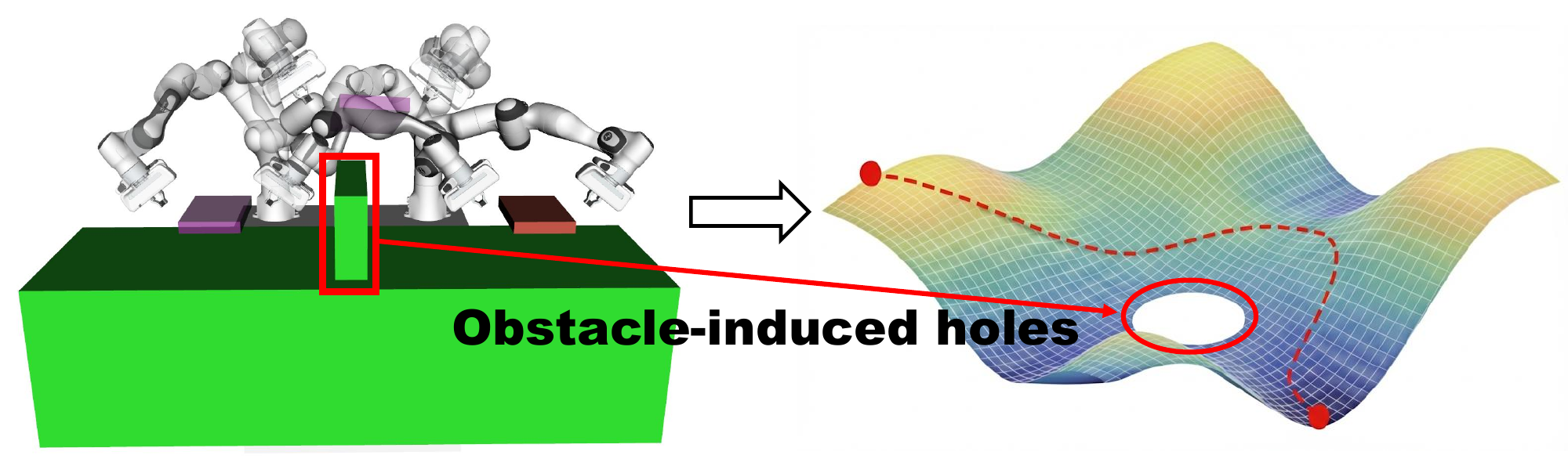}%
\label{fig2_b}}
\caption{Conceptual illustration of inequality-induced geometry on an equality-constrained manifold. (a) Equality-constrained manifold under the implicit assumption of full admissibility, as commonly treated in tangent-space and projection-based planners. (b) The same manifold after removing infeasible regions induced by inequality constraints, which form holes in the feasible subset. Although the equality constraint is unchanged, the admissible region exhibits nontrivial geometry that is not captured by conventional steering. RMRRT accounts for this geometry through a Riemannian metric defined on the tangent space.}
\vspace{-0.1cm}
\label{fig2}
\end{figure*}
This section introduces the notation and geometric preliminaries used throughout the paper. Specifically, we formulate the motion planning problem under equality and inequality constraints, review equality-constrained manifolds and their tangent spaces, and summarize barrier-based representations of inequality constraints used in the metric construction of Section~\ref{section3}.

\noindent\textit{Notation.}
Throughout the paper, $n$ denotes the configuration-space dimension, $m_{\mathrm{eq}}$ the number of equality constraints, $m_{\mathrm{ineq}}$ the number of inequality constraints, and $m_{\mathrm{proxy}}$ the number of geometric proxy points used in the barrier metric. The scalar $\lambda>0$ denotes the metric regularization parameter, and Lagrange multipliers are written as $\nu$ or $\mu$. The equality-constraint Jacobian is denoted by $J(q)\in\mathbb{R}^{m_{\mathrm{eq}}\times n}$, while proxy-point Jacobians are written as $J_j(q)\in\mathbb{R}^{3\times n}$, for $j\in\{1,\dots,m_{\mathrm{proxy}}\}$. A proxy point \(p_j(q) \in \mathbb{R}^3\), \(j \in \{1,\ldots,m_{\mathrm{proxy}}\}\), is a body-fixed point attached to a kinematic link \(L_j\).
Its world position is obtained from forward kinematics as
\begin{equation}
    p_j(q) = T_{L_j}(q)p_j^{\mathrm{local}},
\end{equation}
where \(T_{L_j}(q) \in SE(3)\) is the link transform and \(p_j^{\mathrm{local}}\) is a fixed offset defined in the link frame.
The mapping \(q \mapsto p_j(q)\) introduces no additional state.
The corresponding proxy-point Jacobian is
\begin{equation}
    J_j(q)
    =
    \frac{\partial p_j(q)}{\partial q}
    \in \mathbb{R}^{3 \times n},
\end{equation}
which is the linear part of the manipulator Jacobian evaluated at \(p_j^{\mathrm{local}}\).
Each proxy point carries a signed distance \(d_j(q)\) to the nearest obstacle, computed as an analytic point-to-oriented-bounding-box distance minimized over the obstacle set.
The triple
\begin{equation}
    \{p_j(q), J_j(q), d_j(q)\}
\end{equation}
is the input to the ambient barrier metric constructed in Section \ref{subsec:rbmetric}.

\subsection{Problem Formulation}
We consider motion planning problems for robotic systems operating in a continuous configuration space \(\mathcal{Q} \subset \mathbb{R}^n\), where a configuration is denoted by \(q \in \mathcal{Q}\). The system is subject to both equality and inequality constraints, which jointly define the feasible configuration space. Equality constraints are expressed as
\begin{equation}
\label{eq:equality_constraint}
g(q) = 0,
\end{equation}
where \(g: \mathbb{R}^n \rightarrow \mathbb{R}^{m_{\mathrm{eq}}}\) is a differentiable constraint function. These constraints restrict feasible configurations to an implicitly defined manifold embedded in the ambient configuration space. 

\noindent In addition, inequality constraints are given by
\begin{equation}
\label{eq:inequality_constraint}
h_i(q) > 0, \quad i = 1, \dots, m_{\mathrm{ineq}}
\end{equation}
where each \(h_i: \mathbb{R}^n \rightarrow \mathbb{R}\) represents a state-dependent constraint. The strict inequalities $h_i(q)>0$ define an open feasible set; this is deliberate, because the log-barrier used to construct the metric (Section~\ref{sec:prelim_barrier}) is finite only on the interior $h_i>0$ and diverges as $h_i\to 0$, while hard feasibility is in any case enforced separately by the state- and edge-validity checks. Together, the inequality constraints define admissible regions within the equality-constrained manifold. The overall feasible configuration set is therefore defined as
\begin{equation}
\label{eq:feasible_set}
\mathcal{F} = \left\{ q \in \mathcal{Q} \;\middle|\; g(q) = 0,\; h_i(q) > 0,\; \forall i \right\}.
\end{equation}
We assume that valid start and goal configurations lie within \(\mathcal{F}\), and focus on planning collision-free motions that remain feasible throughout the planning process. In sampling-based motion planning, this requires generating tree expansions that preserve equality constraints while avoiding violations of inequality constraints.

\subsection{Equality-Constrained Manifolds and Tangent Spaces}
\label{sec:prelim_tangent}

Equality constraints of the form \(g(q) = 0\) define an implicit manifold embedded in the ambient configuration space \(\mathcal{Q}\). Specifically, the equality-constrained manifold is given by
\begin{equation}
\mathcal{M}=\{\, q\in\mathcal{Q}\mid g(q)=0 \,\}.
\label{eq:manifold_def}
\end{equation}
Let $J(q)$ denote the Jacobian of the constraint function,
\begin{equation}
J(q)=\frac{\partial g(q)}{\partial q}\in\mathbb{R}^{m_{\mathrm{eq}}\times n}.
\label{eq:jacobian_def}
\end{equation}
If $J(q)$ has full row rank $m_{\mathrm{eq}}$ for all $q\in\mathcal{M}$, then $\mathcal{M}$ is a smooth manifold of dimension $n-m_{\mathrm{eq}}$\footnote{This condition corresponds to the standard regularity assumption of the implicit function theorem.}. At a configuration $q\in\mathcal{M}$, the tangent space of $\mathcal{M}$ is defined as
\begin{equation}
T_q\mathcal{M}=\{\, v\in\mathbb{R}^n \mid J(q)v=0 \,\}.
\label{eq:tangent_space}
\end{equation}
This space characterizes the set of infinitesimal velocities that preserve the equality constraints to first order. In sampling-based motion planning, tangent directions are used to construct local expansions, and a retraction is applied when needed to keep the expanded state on the manifold $\mathcal{M}$. Moreover, when $J(q)$ has full row rank $m_{\mathrm{eq}}$, the matrix $J(q)J(q)^\top$ is invertible, which allows defining the orthogonal projector onto $T_q\mathcal{M}$ as
\begin{equation}
Q(q)=I - J(q)^{\top}\big(J(q)J(q)^{\top}\big)^{-1}J(q),
\label{eq:orth_proj}
\end{equation}
where $I$ denotes the $n\times n$ identity matrix. The projector $Q(q)$ satisfies $J(q)Q(q)=0$ and maps arbitrary vectors in $\mathbb{R}^n$ onto the tangent space $T_q\mathcal{M}$. This tangent-space formulation provides a fundamental tool for constructing constraint-consistent steering and expansion operators, which will be exploited in the proposed approach.

\subsection{Inequality Constraints and Log-Barrier Functions}
\label{sec:prelim_barrier}
As introduced in \eqref{eq:inequality_constraint}, the inequality constraints induce state-dependent boundaries that restrict admissible configurations within the equality-constrained manifold. In sampling-based planning, such constraints are typically handled via binary feasibility checks. To obtain a continuous representation of inequality constraints, we adopt a log-barrier formulation~\cite{boyd2004convex}. Specifically, given a set of $m_{\mathrm{ineq}}$ inequality constraints $\{h_i(q)\}_{i=1}^{m_{\mathrm{ineq}}}$, we define the logarithmic barrier function
\begin{equation}
\phi(q) = - \sum_{i=1}^{m_{\mathrm{ineq}}} \log h_i(q).
\label{eq:log_barrier}
\end{equation}

\noindent This function grows rapidly as a configuration $q$ approaches the boundary of the admissible region, thereby smoothly encoding proximity to inequality constraints. Moreover, the Hessian of the barrier function captures second-order information associated with the curvature of the constraint boundaries, providing a principled means to quantify how inequality constraints locally distort the geometry of the configuration space. This property makes log-barrier functions well-suited for defining geometry-aware structures that can be leveraged during planning, as will be described in Section \ref{section3}.

\section{Riemannian Barrier Metric RRT}
\label{section3}

In this section, we present Riemannian Barrier Metric RRT (RMRRT), a sampling-based planner that combines equality-manifold admissibility with inequality-aware local geometry. We first construct an ambient barrier metric and the induced tangent-space metric, and then show how the resulting geometry is used consistently in steering, nearest-neighbor selection, and manifold correction.

\subsection{Proposed Framework}
\label{sec:rmrrt_framework}

RMRRT builds on a bidirectional RRT-Connect backbone, but replaces purely Euclidean local expansion with a geometry that reflects both equality and inequality constraints. The motivation is illustrated in Fig.~\ref{fig2_a}--\ref{fig2_b}. In the absence of active inequality boundaries, tangent-space steering provides a natural local motion model on the equality-constrained manifold and is widely used in projection-based constrained planners \cite{jang2022motion,suh2011tangent,jaillet2012path}. However, when inequality constraints carve out infeasible regions on the same manifold, tangent-feasible extensions may repeatedly move toward nearby invalid regions, leading to frequent rejection and inefficient exploration.

RMRRT addresses this issue by endowing the equality-consistent
tangent space with a nonuniform geometry induced by nearby inequality
boundaries. Whereas conventional tangent-space planning measures
admissible directions uniformly in the Euclidean sense, RMRRT stretches
the geometry along directions that approach active boundaries. The same
Euclidean displacement is therefore measured as a longer Riemannian
motion when directed toward a nearby boundary, while directions that
preserve clearance remain less distorted. This reshapes local
exploration before an extension is generated rather than relying only
on post-hoc rejection.

RMRRT first constructs an ambient barrier metric $G(q)$ from
inequality-boundary proximity and restricts it to the equality-consistent
tangent space through a $G$-orthogonal projection, yielding the induced
metric $M(q)$. The metric $M(q)$ is used for steering and
nearest-neighbor selection: steering suppresses boundary-directed motion
components, while nearest-neighbor selection evaluates local
connectability in the induced geometry rather than by Euclidean
proximity alone. Thus, a Euclidean-near node may be disfavored if
connecting to it requires motion toward an active boundary.

The ambient metric $G(q)$ additionally defines the
minimum-$G$-norm manifold correction after a finite extension step.
Consequently, steering, node selection, and correction are constructed
consistently under the same local geometry. The following derivations
formalize the well-posedness and geometric consistency of these
operations.

\begin{algorithm}[t]
\caption{\textsc{BuildMetric}$(q)$: Riemannian barrier metric}
\label{alg:rbmetric}
\begin{algorithmic}[1]
\REQUIRE Configuration $q$
\ENSURE Ambient barrier metric $G(q)$, tangent projector $P(q)$, induced tangent-space metric $M(q)$

\FOR{$j=1,\dots,m_{\mathrm{proxy}}$}
    \STATE Compute proxy point $p_j(q)$ and signed distance $d_j(q)$ \COMMENT{proxy kinematics and clearance}
    \STATE Compute point Jacobian $J_j(q)=\partial p_j(q)/\partial q$
\ENDFOR
\STATE $G(q) \leftarrow \lambda I + \sum_{j=1}^{m_{\mathrm{proxy}}} w_j(q)\,J_j^\top(q)J_j(q)$ \COMMENT{ambient barrier metric~\eqref{eq:G_GN}, weights~\eqref{eq:wj_def}}
\STATE Compute equality Jacobian $J(q)$
\STATE $A \leftarrow J(q)\,G^{-1}(q)\,J^\top(q)$
\STATE $P(q) \leftarrow I - G^{-1}(q)J^\top(q)\,A^{-1}\,J(q)$ \COMMENT{$G$-orthogonal tangent projector~\eqref{eq:proj}}
\STATE $M(q) \leftarrow P^\top(q)\,G(q)\,P(q)$ \COMMENT{induced tangent-space metric~\eqref{eq:M}}
\RETURN $G(q),P(q),M(q)$
\end{algorithmic}
\end{algorithm}

\subsection{Riemannian Barrier Metric Construction}
\label{subsec:rbmetric}

We first construct an ambient barrier metric that captures how proximity to inequality boundaries reshapes local geometry. In principle, this metric could be formed from the exact Hessians of the inequality barrier terms:
\begin{equation}
G(q)= I + \sum_{i=1}^{m_{\mathrm{ineq}}}\alpha_i \nabla^2 \phi_i(q),
\label{eq:G_ideal}
\end{equation}
where $\phi_i$ is the log-barrier associated with the $i$-th inequality constraint and $\alpha_i>0$ is a weighting factor. This ideal form directly reflects boundary-induced curvature, assigning larger metric cost to motions that reduce constraint margins. In practice, however, exact barrier Hessians are expensive and numerically fragile in cluttered scenes, since signed-distance fields may become nonsmooth under closest-feature changes and contact switching. We therefore adopt a Gauss--Newton (GN) surrogate~\cite{nocedal2006numerical} that retains the dominant positive semidefinite contribution of the barrier curvature. For a distance-based inequality constraint with signed distance $d(q)$ and barrier $\phi(q)=\phi(d(q))$, the chain rule gives
\begin{equation}
\nabla^2 \phi(q)
=
\phi''(d)\,\nabla d(q)\nabla d(q)^\top
+
\phi'(d)\,\nabla^2 d(q).
\label{eq:gn_chain_rule}
\end{equation}
The second term depends on the curvature of the signed-distance field and may be ill-conditioned or undefined near feature switching. We therefore retain only the dominant positive semidefinite term and use the approximation
\begin{equation}
\nabla^2 \phi(q)
\approx
\phi''(d)\,\nabla d(q)\nabla d(q)^\top.
\label{eq:gn_psd_approx}
\end{equation}

\noindent Now let $p(q)\in\mathbb{R}^3$ be a geometric proxy point with Jacobian
$J_p(q)=\partial p(q)/\partial q\in\mathbb{R}^{3\times n}$, and let
$n(q)=\partial d/\partial p\in\mathbb{R}^3$ denote the local distance normal. Then
\begin{equation}
\nabla d(q)=n(q)^\top J_p(q),
\label{eq:grad_d_proxy}
\end{equation}
which yields
\begin{equation}
\nabla^2 \phi(q)
\approx
\phi''(d)\,J_p(q)^\top \big(n(q)n(q)^\top\big)J_p(q).
\label{eq:gn_proxy_hessian}
\end{equation}
Since $n(q)n(q)^\top \preceq I$, we obtain the upper bound
\begin{equation}
\nabla^2 \phi(q)
\preceq
\phi''(d)\,J_p(q)^\top J_p(q).
\label{eq:gn_outer_bound}
\end{equation}

Motivated by this bound, we replace the exact directional term with a stable positive semidefinite surrogate and absorb the barrier curvature into a nonnegative distance-dependent activation weight. Summing over $m_{\mathrm{proxy}}$ proxy points leads to the practical ambient barrier metric
\begin{equation}
G(q)= \lambda I + \sum_{j=1}^{m_{\mathrm{proxy}}} w_j(q)\,J_j(q)^\top J_j(q),
\label{eq:G_GN}
\end{equation}
where $\lambda>0$ is a regularizer, $d_j(q)$ is the signed distance associated with the $j$-th proxy point, $J_j(q)\in\mathbb{R}^{3\times n}$ is the corresponding point Jacobian, and $w_j(q)$ is a distance-based activation weight. The practical metric \eqref{eq:G_GN} is instantiated from distance- and proximity-based geometric inequalities---namely robot--obstacle proximity---for which signed distances and proxy-point Jacobians provide a stable and inexpensive positive-semidefinite surrogate. Inequalities without such tractable differential structure are not embedded in the metric; instead, they are enforced separately through the state- and edge-validity checks used in the overall planner (Section~\ref{subsec:overall}). The metric thus serves as an exploration bias rather than a feasibility certificate, while hard feasibility is enforced independently by full-body state- and edge-validity checking at the interpolation resolution used in our implementation.

\begin{definition}[Barrier Activation Weight]
\label{def:barrier_weight}
For the $j$-th proxy point with signed distance $d_j(q)$ to the nearest obstacle, define
\begin{equation}
w_j(q)=\alpha_j\,\psi\!\left(\bar d_j(q)\right)\,\mathbf{1}[d_j(q)<\delta_{\mathrm{act}}],
\label{eq:wj_def}
\end{equation}
where
\begin{equation}
\bar d_j(q)=\max\!\big(d_j(q)-\delta_{\mathrm{safe}},\,0\big),
\label{eq:dbar}
\end{equation}
$\alpha_j>0$ is a category-dependent scaling factor, $\delta_{\mathrm{safe}}\ge 0$ is a safety margin, $\delta_{\mathrm{act}}>0$ is the activation radius, and
\begin{equation}
\psi(d)=\min\!\left(\frac{1}{d^2+\epsilon_d^2},\,\psi_{\max}\right).
\label{eq:psi_def}
\end{equation}
Here, $\epsilon_d>0$ avoids singular behavior near contact, and $\psi_{\max}$ caps the metric contribution for numerical stability.
\end{definition}

\begin{remark}[Barrier-Kernel Interpretation]
\label{rem:kernel_behavior}
The kernel $\psi(d)$ is a regularized and saturated approximation of the dominant $1/d^2$ behavior of the barrier curvature. As $d\to\infty$, $\psi(d)\to 0$ and the metric approaches $\lambda I$, recovering near-isotropic exploration. As $d\to 0$, $\psi(d)$ saturates at $\psi_{\max}$, ensuring bounded metric contribution near contact. Thus, the metric smoothly transitions from near-isotropic exploration far from obstacles to anisotropic penalization near active inequality boundaries.
\end{remark}
\vspace{-0.3cm}
\begin{remark}[Non-distance Differentiable Inequality Margins]
\label{remark2}
Although the metric evaluated in this work is constructed from signed distances and proxy-point Jacobians, RMRRT is not structurally restricted to this form. Any differentiable inequality margin $h(q)>0$ that yields a symmetric positive-semidefinite metric contribution can be incorporated into the ambient metric $G(q)$. For example, a manipulability margin may be defined as
\begin{equation*}
h_{\mathrm{man}}(q)=\log\det\!\big(J_r(q)J_r(q)^\top+\epsilon I\big)-s_{\min},
\end{equation*}
where $J_r(q)$ is a task or constraint Jacobian, $\epsilon>0$ is a damping constant, and $s_{\min}$ is a prescribed lower bound. The corresponding contribution to the ambient metric can be written as
\begin{equation*}
G_{\mathrm{man}}(q)=w_{\mathrm{man}}(q)\,\nabla h_{\mathrm{man}}(q)\,\nabla h_{\mathrm{man}}(q)^\top,
\end{equation*}
where $w_{\mathrm{man}}(q)$ uses the same regularized and saturated activation pattern as the distance-based weights. Appendix~\ref{app:manip} provides a qualitative proof-of-concept using this term in a near-singular reaching task. This example illustrates framework-level extensibility, while the main empirical claims remain restricted to signed-distance-based geometric proxy inequalities.
\end{remark}

\vspace{-0.3cm}
\begin{remark}[Proxy Points Are Not Collision Detectors]
\label{rem:proxy_not_detector}
The proxy points and their clearances $d_j(q)$ define only where the barrier metric senses inequality-induced local geometry; they do not certify collision-freeness and need not cover the free space $C_{\mathrm{free}}$. Hard admissibility is enforced separately by \textsc{StateValid} and \textsc{EdgeValid}, which apply the full-resolution collision detector to all robot links, so every infeasible state or edge---including those in cluttered, non-convex, multi-cuboid scenes---is rejected regardless of proxy placement. Because each $d_j(q)$ is recomputed per query as the minimum distance to all obstacles and the metric uses only this local value and its gradient direction, no globally accurate or convex distance field is required. Consequently the soundness of RMRRT is independent of proxy density and distribution, which affect only the informativeness of the bias and hence exploration efficiency. The proxy representation can therefore be replaced or refined---dense surface sampling, analytic SDFs, robot SDFs, or learned SDF/RDF models---without changing the pipeline or its guarantees.
\end{remark}

To make this geometry compatible with equality constraints, we restrict the ambient barrier metric to the tangent space of the equality-constrained manifold $\mathcal{M}$. Recall that
\begin{equation}
T_q\mathcal{M}
=
\left\{
v\in\mathbb{R}^n
\mid
J(q)v=0
\right\}.
\label{eq:tangent_repeat}
\end{equation}

\noindent Given any ambient vector $w\in\mathbb{R}^n$, its $G$-orthogonal projection onto $T_q\mathcal{M}$ is defined by
\begin{equation}
\min_{v\in\mathbb{R}^n}
\frac{1}{2}(v-w)^\top G(q)(v-w)
\quad
\text{s.t.}
\quad
J(q)v=0.
\label{eq:proj_problem}
\end{equation}

\noindent The corresponding KKT conditions~\cite{boyd2004convex} are
\begin{equation}
G(q)(v-w)+J(q)^\top \nu=0,
\qquad
J(q)v=0.
\label{eq:proj_kkt}
\end{equation}
From the stationarity condition, we obtain
\begin{equation}
G(q)v
=
G(q)w-J(q)^\top\nu,
\end{equation}
and hence
\begin{equation}
v
=
w-G(q)^{-1}J(q)^\top\nu.
\end{equation}
Enforcing the feasibility condition $J(q)v=0$ gives
\begin{equation}
J(q)w
-
J(q)G(q)^{-1}J(q)^\top\nu
=
0,
\end{equation}
or equivalently,
\begin{equation}
J(q)w
=
J(q)G(q)^{-1}J(q)^\top\nu.
\end{equation}
Therefore, the Lagrange multiplier is given by
\begin{equation}
\nu
=
\left(
J(q)G(q)^{-1}J(q)^\top
\right)^{-1}
J(q)w.
\end{equation}
Since $G(q)\succ0$ and $J(q)$ has full row rank,
$J(q)G(q)^{-1}J(q)^\top$ is invertible. Substituting this expression back into $v$ yields
\begin{equation}
v
=
w
-
G(q)^{-1}J(q)^\top
\left(
J(q)G(q)^{-1}J(q)^\top
\right)^{-1}
J(q)w.
\end{equation}
Factoring out $w$, we obtain
\begin{equation}
v
=
\left(
I
-
G(q)^{-1}J(q)^\top
\left(
J(q)G(q)^{-1}J(q)^\top
\right)^{-1}
J(q)
\right)w.
\end{equation}
Defining the linear operator in parentheses as $P(q)$, we have
\begin{equation}
v=P(q)w,
\label{eq:proj_vec}
\end{equation}
with
\begin{equation}
P(q)
=
I
-
G(q)^{-1}J(q)^\top
\left(
J(q)G(q)^{-1}J(q)^\top
\right)^{-1}
J(q).
\label{eq:proj}
\end{equation}

\noindent The induced tangent-space metric is then defined as
\begin{equation}
M(q)
=
P(q)^\top G(q)P(q).
\label{eq:M}
\end{equation}
This induced tangent-space metric is used in all subsequent planner operations.

\begin{remark}[Regularity and Singular Configurations]
This construction assumes $J(q)$ has full row rank, so that the Schur complement $S(q)=J(q)G(q)^{-1}J(q)^\top$ is nonsingular and the projector~\eqref{eq:proj}, the steering direction, and the minimum-$G$-norm correction are well defined. At a singular configuration, $J(q)$ loses rank and $S(q)$ becomes singular, so the projector~\eqref{eq:proj} and correction~\eqref{eq:riemannian_newton_projection} are no longer uniquely defined; a small $\sigma_{\min}(J(q))$ also makes these solves ill-conditioned. Proposition~\ref{prop:probabilistic_completeness} therefore holds only on regular regions bounded away from rank deficiency, and RMRRT does not guarantee traversal of kinematic singularities---a regularity assumption shared by projection-, atlas-, and tangent-space-based constrained planners.
\end{remark}

\begin{figure*}[!t]
\centering
\subfloat[]{\includegraphics[width=3.5in]{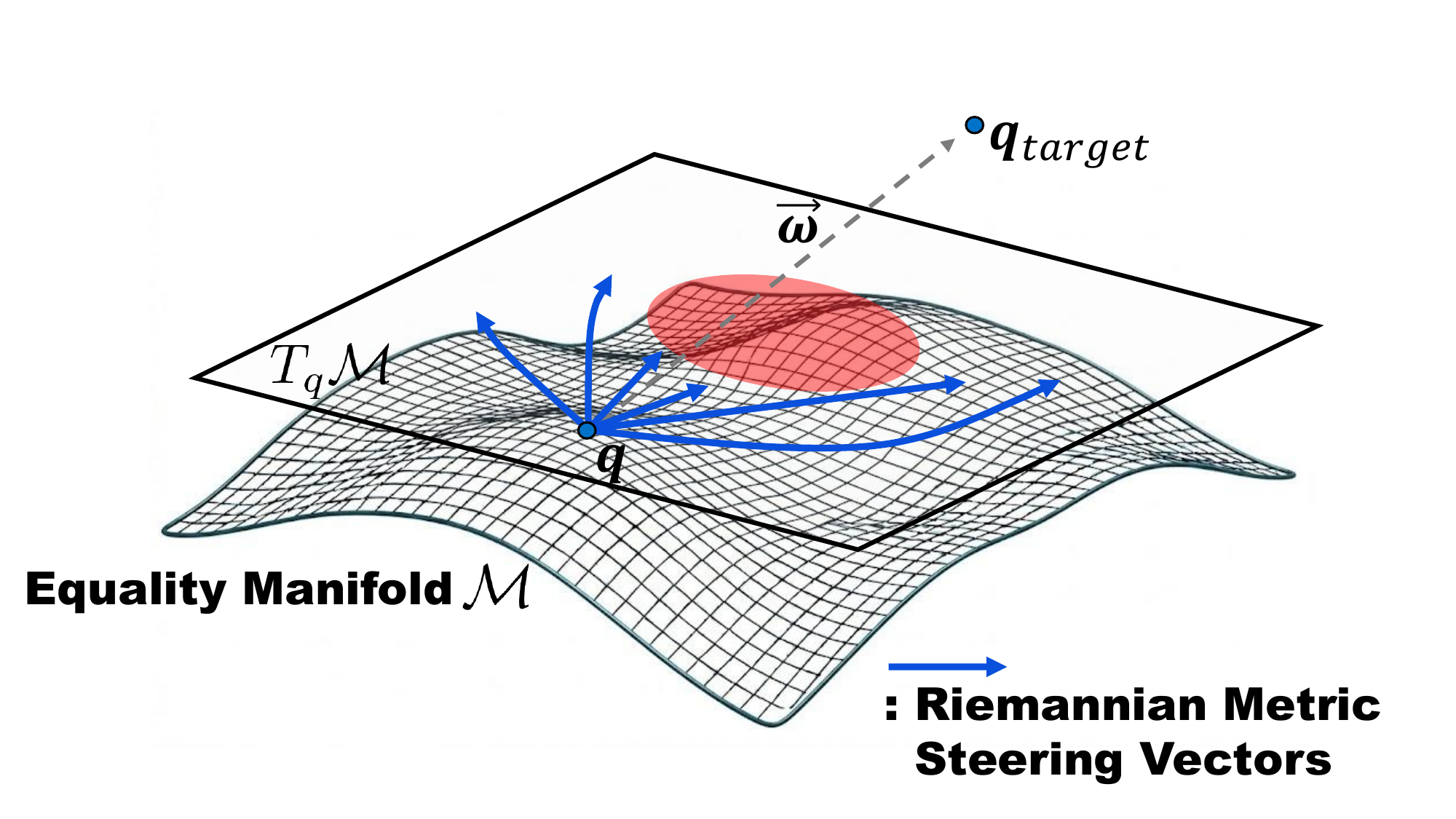}%
\label{fig3_a_case}}
\hfil
\subfloat[]{\includegraphics[width=3.5in]{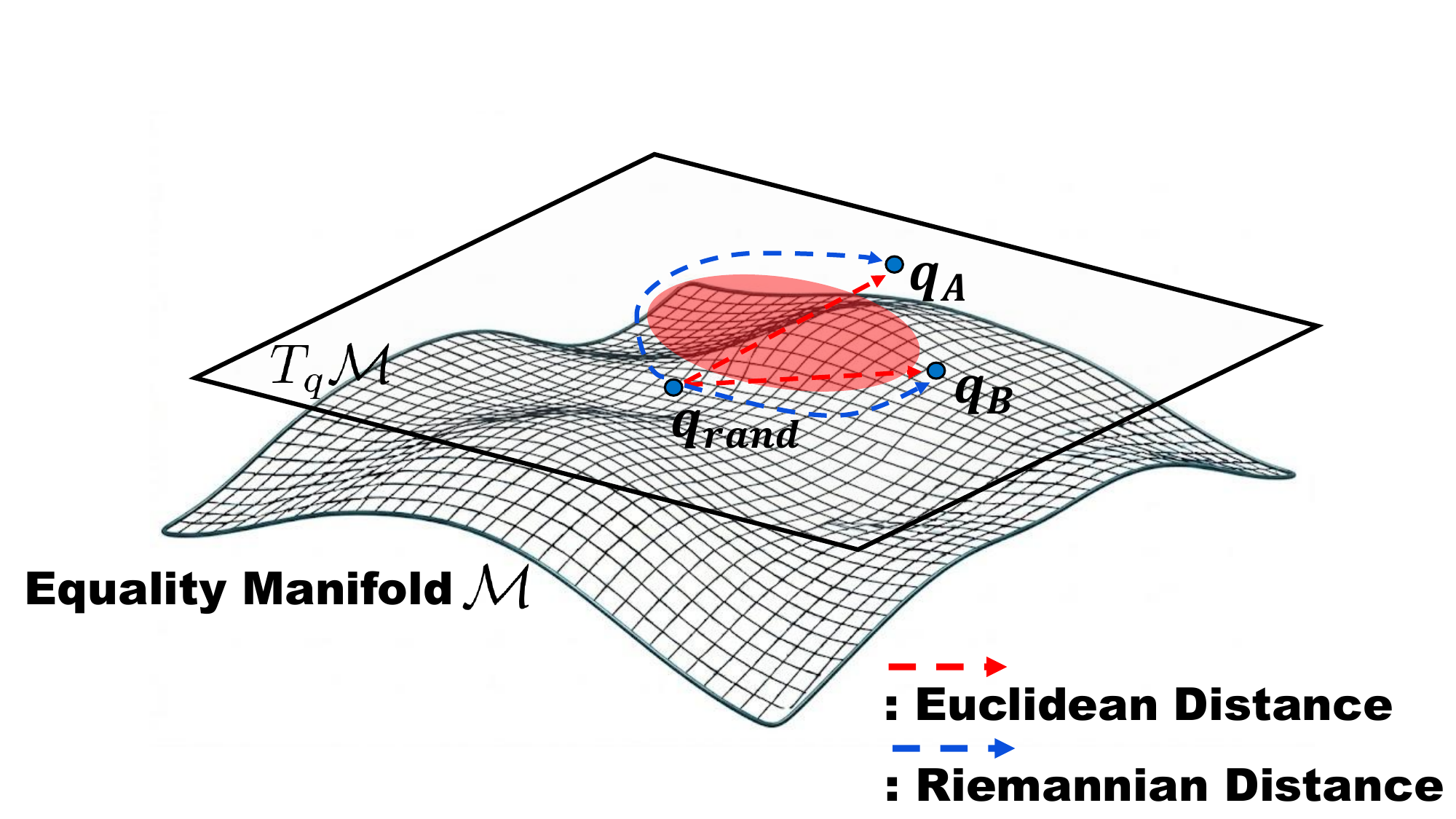}%
\label{fig3_b_case}}
\caption{Conceptual illustration of the metric-aware exploration mechanism in RMRRT.
(a) Metric-induced tangent-space steering under the induced tangent-space metric.
(b) Metric-aware nearest-neighbor selection under local inequality-dependent anisotropy.}
\label{fig_manifold}
\end{figure*}

\begin{lemma}[Tangent-Space Identity of the \(G\)-Orthogonal Projector]
\label{lem:tangent_identity}
Let $P(q)$ be defined as in \eqref{eq:proj} and suppose $J(q)$ has full row rank. Then, for any $v\in T_q\mathcal{M}$,
\begin{equation}
P(q)v=v.
\label{eq:tangent_identity}
\end{equation}
\end{lemma}

\begin{proof}
Expanding \eqref{eq:proj} gives
\begin{equation}
P(q)v
=
v-
G^{-1}(q)J(q)^\top
\big(J(q)G^{-1}(q)J(q)^\top\big)^{-1}J(q)v.
\label{eq:Pv_expand}
\end{equation}
Since $v\in \mathrm{Null}(J(q))$, we have $J(q)v=0$, and therefore $P(q)v=v$.
\end{proof}

\begin{lemma}[Metric Equivalence on the Tangent Space]
\label{lem:metric_equivalence}
For any tangent-projected displacement
$\tilde r=P(q)(q_{\mathrm{rand}}-q)\in T_q\mathcal{M}$,
the quadratic form induced by $M(q)$ satisfies
\begin{equation}
\tilde r^\top M(q)\tilde r=\tilde r^\top G(q)\tilde r.
\label{eq:metric_equivalence}
\end{equation}
\end{lemma}

\begin{proof}
Since $M(q)=P(q)^\top G(q)P(q)$ and $\tilde r\in T_q\mathcal{M}$, Lemma~\ref{lem:tangent_identity} implies $P(q)\tilde r=\tilde r$. Therefore,
\[
\tilde r^\top M(q)\tilde r
=
\tilde r^\top P(q)^\top G(q)P(q)\tilde r
=
\tilde r^\top G(q)\tilde r.
\]
\end{proof}

\begin{remark}[Metric Evaluation after Projection]
\label{rem:metric_eval}
Lemma~\ref{lem:metric_equivalence} shows that once a displacement has been projected into $T_q\mathcal{M}$, metric-aware distance evaluation can be carried out using $G(q)$, without explicitly forming $P(q)^\top G(q)P(q)$ for each candidate.
\end{remark}

\begin{proposition}[Adaptive Metric Distortion Bound]
\label{prop:metric_distortion}
For any feasible configuration $q\in\mathcal{F}$ with clearance
\begin{equation}
\delta_{\min}=\min_j d_j(q)>0,
\label{eq:delta_min}
\end{equation}
the condition number of the induced tangent-space metric restricted to $T_q\mathcal{M}$ satisfies
\begin{equation}
\kappa\!\left(M(q)\big|_{T_q\mathcal{M}}\right)
\le
1+\frac{m_{\mathrm{proxy}}\,\alpha_{\max}\,\psi(\bar{\delta}_{\min})\,C_J}{\lambda},
\label{eq:kappa_bound}
\end{equation}
where
\begin{equation}
C_J=\max_j \|J_j(q)^\top J_j(q)\|_2,
\qquad
\alpha_{\max}=\max_j \alpha_j,
\label{eq:CJ_alphamax}
\end{equation}
and
\begin{equation}
\bar{\delta}_{\min}
=
\max\!\big(\delta_{\min}-\delta_{\mathrm{safe}},\,0\big).
\label{eq:deltabar_min}
\end{equation}
\end{proposition}

\begin{proof}
Let $v\in T_q\mathcal{M}$ with $\|v\|=1$. By Lemma~\ref{lem:tangent_identity},
$P(q)v=v$, and by Lemma~\ref{lem:metric_equivalence},
$v^\top M(q)v=v^\top G(q)v$. Using \eqref{eq:G_GN} and Definition~\ref{def:barrier_weight},
\begin{equation}
v^\top G(q)v
=
\lambda
+
\sum_{j=1}^{m_{\mathrm{proxy}}}
w_j(q)\|J_j(q)v\|^2.
\label{eq:metric_quadratic_form}
\end{equation}
Each summand is nonnegative, which gives the lower bound $v^\top G(q)v\ge\lambda$. For the upper bound, $d_j(q)\ge\bar{\delta}_{\min}$ for all active proxy points and $\psi$ is non-increasing; hence,
\begin{align}
w_j(q)
&\le
\alpha_{\max}\psi(\bar{\delta}_{\min}),\\
\|J_j(q)v\|^2
&=
v^\top J_j(q)^\top J_j(q)v\\
&\le
\|J_j(q)^\top J_j(q)\|_2
\le
C_J.
\end{align}
Therefore, each proxy contribution is bounded by
$\alpha_{\max}\psi(\bar{\delta}_{\min})C_J$. Summing over the
$m_{\mathrm{proxy}}$ terms gives
\begin{equation}
\lambda
\le
v^\top G(q)v
\le
\lambda
+
m_{\mathrm{proxy}}
\alpha_{\max}
\psi(\bar{\delta}_{\min})
C_J.
\label{eq:metric_quadratic_bound}
\end{equation}
The claimed condition-number bound follows by taking the ratio of the
upper and lower Rayleigh-quotient bounds.
\end{proof}

The construction above yields an induced tangent-space metric $M(q)$ that preserves admissible first-order motions while encoding proximity to active inequality boundaries. In particular, Proposition~\ref{prop:metric_distortion} shows that the resulting metric distortion remains controlled by obstacle clearance and regularization. The next two subsections use this same metric for local steering and nearest-neighbor selection.

\begin{algorithm}[t]
\caption{\textsc{MetricSteer}$(q, q_{\mathrm{target}}, \eta)$: metric-induced tangent-space steering}
\label{alg:metricsteer}
\begin{algorithmic}[1]
\REQUIRE Feasible node $q\in\mathcal{M}$, target sample $q_{\mathrm{target}}$, step size $\eta>0$
\ENSURE New configuration $q_{\mathrm{new}}$

\STATE $w \leftarrow q_{\mathrm{target}} - q$
\STATE $(G,P,M) \leftarrow \textsc{BuildMetric}(q)$
\STATE $\tilde w \leftarrow P\,w$ \COMMENT{tangent projection~\eqref{eq:steer_w_tilde}}
\IF{$\|\tilde w\| < \epsilon$}
    \RETURN $q$
\ENDIF
\STATE Solve $M\,v_{\mathrm{tmp}}=\tilde w$ \COMMENT{closed-form steering direction~\eqref{eq:steer_closedform}}
\STATE $v_0 \leftarrow P\,v_{\mathrm{tmp}}$ \COMMENT{re-projection~\eqref{eq:steer_v0}}
\STATE $s \leftarrow \sqrt{v_0^\top M\,v_0}$
\IF{$s < \epsilon$}
  \RETURN $q$
\ENDIF
\STATE $\hat v \leftarrow v_0 / s$
\STATE $q_{\mathrm{new}} \leftarrow q + \eta\,\hat v$ \COMMENT{metric normalization and fixed step~\eqref{eq:steer_step}}
\RETURN $q_{\mathrm{new}}$
\end{algorithmic}
\end{algorithm}

\subsection{Metric-Induced Tangent-Space Steering}

Given the induced tangent-space metric $M(q)$, we define a local steering rule that selects a tangent direction making first-order progress toward the target while minimizing local metric energy. As illustrated in Fig.~\ref{fig3_a_case}, the resulting direction remains equality-consistent to first order while being biased away from nearby active inequality boundaries. For a tree node $q\in\mathcal{M}$, let $q_{\mathrm{target}}\in\mathcal{Q}\subset\mathbb{R}^n$ denote a target sample drawn from the ambient configuration space. The sample is used only to define the ambient displacement
\begin{equation}
w = q_{\mathrm{target}} - q.
\label{eq:steer_w}
\end{equation}
Its tangent-space component is
\begin{equation}
\tilde w = P(q)w \in T_q\mathcal{M}. 
\label{eq:steer_w_tilde}
\end{equation}
Thus, $q_{\mathrm{target}}$ itself need not satisfy the equality constraints; only the tangent component of the ambient displacement is used for steering. The finite-step state generated from this direction is subsequently corrected back onto $\mathcal{M}$.

\begin{proposition}[Well-Posedness of Metric-Induced Steering]
\label{prop:steer_wellposed}
Assume that $J(q)$ has full row rank and that $G(q)\succ 0$. Then the projector $P(q)$ is well-defined, the induced tangent-space metric $M(q)$ is positive definite on $T_q\mathcal{M}$, and for any nonzero $\tilde w \in T_q\mathcal{M}$, the optimization problem
\begin{equation}
v^\star = \arg\min_{v \in T_q\mathcal{M}} \frac{1}{2} v^\top M(q)v
\quad \text{s.t.} \quad v^\top \tilde w = 1
\label{eq:steer_opt}
\end{equation}
admits a unique minimizer.
\end{proposition}

\begin{proof}
Since $G(q)\succ 0$, the Schur complement
\begin{equation}
S(q)=J(q)G^{-1}(q)J(q)^\top
\label{eq:steer_schur}
\end{equation}
is positive definite under the full-row-rank assumption on $J(q)$, so $P(q)$ is well-defined. Moreover, for any nonzero $v \in T_q\mathcal{M}$, Lemma~\ref{lem:tangent_identity} implies $P(q)v=v$, hence
\[
v^\top M(q)v = v^\top G(q)v > 0.
\]
Therefore $M(q)$ is positive definite on $T_q\mathcal{M}$. The objective in \eqref{eq:steer_opt} is strictly convex on the tangent space, and the affine constraint is nonempty whenever $\tilde w \neq 0$, so the minimizer is unique.
\end{proof}

Introducing a scalar Lagrange multiplier $\nu$, define
\begin{equation}
L(v,\nu)=\frac{1}{2}v^\top M(q)v + \nu(1-v^\top \tilde w).
\label{eq:steer_lagrangian}
\end{equation}
The first-order optimality conditions yield
\begin{equation}
v^\star=\frac{M^{-1}(q)\tilde w}{\tilde w^\top M^{-1}(q)\tilde w}.
\label{eq:steer_closedform}
\end{equation}
Equation~\eqref{eq:steer_closedform} shows that the projected displacement is reweighted by the inverse induced tangent-space metric, amplifying inexpensive directions while suppressing directions toward nearby active boundaries.

For numerical robustness, we optionally reapply the tangent projector and define
\begin{equation}
v_0 = P(q)M^{-1}(q)\tilde w.
\label{eq:steer_v0}
\end{equation}
We then normalize this direction with respect to $M(q)$ and take a fixed-length step:
\begin{equation}
\hat v = \frac{v_0}{\sqrt{v_0^\top M(q)v_0}},
\qquad
q_{\mathrm{new}} = q + \eta \hat v,
\label{eq:steer_step}
\end{equation}
where $\eta>0$ is the step size. The complete procedure is summarized in Algorithm~\ref{alg:metricsteer}.

\begin{algorithm}[t]
\caption{\textsc{MetricNearest}$(q_{\mathrm{rand}}, \mathcal{V}, K)$: metric-aware nearest-neighbor selection}
\label{alg:metricnearest}
\begin{algorithmic}[1]
\REQUIRE Sample $q_{\mathrm{rand}}$, tree nodes $\mathcal{V}\subset\mathcal{M}$, candidate size $K$
\ENSURE Nearest node $q_{\mathrm{near}}$

\STATE $\mathcal{V}_K \leftarrow \textsc{EuclideanKNN}(q_{\mathrm{rand}}, \mathcal{V}, K)$ \COMMENT{Euclidean pre-selection, adaptive $K$~\eqref{eq:K_adaptive}}
\STATE $q_{\mathrm{near}} \leftarrow \emptyset,\; d_{\min} \leftarrow +\infty$
\FORALL{$q \in \mathcal{V}_K$}
    \STATE $(G,P,M) \leftarrow \textsc{BuildMetric}(q)$ \COMMENT{local metric (Alg.~\ref{alg:rbmetric})}
    \STATE $r \leftarrow q_{\mathrm{rand}} - q$
    \STATE $\tilde r \leftarrow P\, r$
    \STATE $d \leftarrow \sqrt{\tilde r^\top M\, \tilde r}$ \COMMENT{induced metric distance~\eqref{eq:metric_distance}}
    \IF{$d < d_{\min}$}
        \STATE $d_{\min} \leftarrow d,\;\; q_{\mathrm{near}} \leftarrow q$ \COMMENT{metric reranking~\eqref{eq:nearest_metric_subset}}
    \ENDIF
\ENDFOR
\RETURN $q_{\mathrm{near}}$
\end{algorithmic}
\end{algorithm}

\begin{algorithm}[t]
\caption{Overall RMRRT Framework}
\label{alg:rmrrt}
\begin{algorithmic}[1]
\REQUIRE Feasible start/goal configurations $q_{\mathrm{start}}, q_{\mathrm{goal}} \in \mathcal{M}$, step size $\eta$, candidate size $K$, maximum iterations $N_{\max}$
\ENSURE Feasible path $\sigma$ or \textsc{Failure}

\STATE Initialize two trees $\mathcal{T}_a \leftarrow \{q_{\mathrm{start}}\}$, $\mathcal{T}_b \leftarrow \{q_{\mathrm{goal}}\}$
\FOR{$n = 1$ to $N_{\max}$}
\STATE $q_{\mathrm{rand}} \sim \textsc{UniformSample}(\mathcal{Q})$
\STATE $q_{\mathrm{near}} \leftarrow \textsc{MetricNearest}(q_{\mathrm{rand}}, \mathcal{V}(\mathcal{T}_a), K)$ \COMMENT{metric-aware nearest neighbor (Alg.~\ref{alg:metricnearest},~\eqref{eq:nearest_metric_subset})}
\STATE $q_{\mathrm{new}} \leftarrow \textsc{MetricSteer}(q_{\mathrm{near}}, q_{\mathrm{rand}}, \eta)$ \COMMENT{metric-induced steering (Alg.~\ref{alg:metricsteer})}
\STATE $isStateValid \leftarrow \textsc{StateValid}(q_{\mathrm{new}})$ \COMMENT{hard feasibility enforced here (state/edge checks)}
\STATE $isEdgeValid \leftarrow \textsc{EdgeValid}(q_{\mathrm{near}}, q_{\mathrm{new}})$
\IF{\textbf{not} $isStateValid$ \OR \textbf{not} $isEdgeValid$}
    \STATE \textbf{continue}
\ENDIF
\STATE Add $q_{\mathrm{new}}$ to $\mathcal{T}_a$ with parent $q_{\mathrm{near}}$
\STATE $q_{\mathrm{curr}} \leftarrow \textsc{MetricNearest}(q_{\mathrm{new}}, \mathcal{V}(\mathcal{T}_b), K)$ \COMMENT{connect opposite tree (bidirectional backbone)}
\REPEAT
    \STATE $q_{\mathrm{next}} \leftarrow \textsc{MetricSteer}(q_{\mathrm{curr}}, q_{\mathrm{new}}, \eta)$
    \STATE $isStateValid \leftarrow \textsc{StateValid}(q_{\mathrm{next}})$
    \STATE $isEdgeValid \leftarrow \textsc{EdgeValid}(q_{\mathrm{curr}}, q_{\mathrm{next}})$
    \IF{\textbf{not} $isStateValid$ \OR \textbf{not} $isEdgeValid$}
        \STATE \textbf{break}
    \ENDIF
    \STATE Add $q_{\mathrm{next}}$ to $\mathcal{T}_b$ with parent $q_{\mathrm{curr}}$
    \STATE $q_{\mathrm{curr}} \leftarrow q_{\mathrm{next}}$
\UNTIL{$\|q_{\mathrm{curr}} - q_{\mathrm{new}}\| < \epsilon_{\mathrm{conn}}$}
\IF{$\|q_{\mathrm{curr}} - q_{\mathrm{new}}\| < \epsilon_{\mathrm{conn}}$}
    \STATE $\sigma \leftarrow \textsc{ExtractPath}(\mathcal{T}_a, q_{\mathrm{new}}, \mathcal{T}_b, q_{\mathrm{curr}})$
    \RETURN $\sigma$
\ENDIF
\STATE Swap $(\mathcal{T}_a, \mathcal{T}_b)$ \COMMENT{swap active tree}
\ENDFOR
\RETURN \textsc{Failure}
\end{algorithmic}
\end{algorithm}
\subsection{Metric-Aware Nearest-Neighbor Selection}
\label{subsec:nearest}

We next use the same induced tangent-space metric $M(q)$ to define a geometry-consistent nearest-neighbor rule. Rather than ranking tree nodes solely by ambient Euclidean proximity, the planner evaluates the local cost of the tangent-projected displacement under $M(q)$. As illustrated in Fig.~\ref{fig3_b_case}, this anisotropic geometry favors nodes that are locally well-connected under the induced metric rather than those that are merely Euclidean-nearest. Given an ambient sample $q_{\mathrm{rand}}\in\mathcal{Q}$ and a set of tree nodes $\mathcal{V}\subset\mathcal{M}$, define
\begin{equation}
d_M(q,q_{\mathrm{rand}})
=
\big\|P(q)(q_{\mathrm{rand}}-q)\big\|_{M(q)}.
\label{eq:metric_distance}
\end{equation}
Here, $\|x\|_{M(q)} := \sqrt{x^{T}M(q)x}$ denotes the norm induced by the tangent-space metric $M(q)$. The nearest node is then selected as
\begin{equation}
q_{\mathrm{near}}
=
\arg\min_{q\in\mathcal{V}}
\big\|P(q)(q_{\mathrm{rand}}-q)\big\|_{M(q)}.
\label{eq:nearest_metric_full}
\end{equation}
Although $q_{\mathrm{rand}}$ need not lie on $\mathcal{M}$, the displacement is projected onto $T_q\mathcal{M}$ before evaluation. Hence, the comparison depends only on locally equality-consistent first-order motions. Moreover, by Lemma~\ref{lem:metric_equivalence}, the quadratic form can be evaluated using the ambient barrier metric $G(q)$ once the projection has been applied. A full metric-aware search over all nodes would be unnecessarily expensive. We therefore use Euclidean preselection followed by metric reranking. Let $\mathcal{V}_K\subset\mathcal{V}$ denote the set of the $K$ Euclidean-nearest candidates to $q_{\mathrm{rand}}$, and perform reranking only over this subset:
\begin{equation}
q_{\mathrm{near}}
=
\arg\min_{q\in\mathcal{V}_K}
\big\|P(q)(q_{\mathrm{rand}}-q)\big\|_{M(q)}.
\label{eq:nearest_metric_subset}
\end{equation}

\begin{remark}[Consistency of Euclidean Preselection]
\label{rem:euclidean_preselection}
The two-stage strategy in Algorithm~\ref{alg:metricnearest} is consistent with metric-aware selection when $K$ is sufficiently large, because Proposition~\ref{prop:metric_distortion} implies norm equivalence between the Euclidean norm and the induced tangent-space metric on the tangent space. Specifically, for any feasible node $q\in\mathcal{F}$ with clearance $\delta_{\min}>0$ and any tangent-projected displacement $\tilde r=P(q)(q_{\mathrm{rand}}-q)$,
\begin{equation}
\sqrt{\lambda}\,\|\tilde r\|
\;\le\;
\|\tilde r\|_{M(q)}
\;\le\;
\sqrt{\Lambda}\,\|\tilde r\|,
\label{eq:bilipschitz_preselection}
\end{equation}
where $\lambda$ and $\Lambda$ are the spectral bounds implied by Proposition~\ref{prop:metric_distortion}.
\end{remark}

\begin{remark}[Asymmetry of the Metric-Aware Distance]
\label{rem:asymmetric_metric_distance}
The distance $d_M(q,q_{\mathrm{rand}})$ is evaluated using the local metric at the tree node $q$, not at the sampled configuration $q_{\mathrm{rand}}$. This is intentional: the metric at $q$ encodes the local geometry of the region from which expansion occurs, ensuring consistency between nearest-neighbor selection and subsequent steering.
\end{remark}

In our implementation, the candidate size is chosen adaptively as
\begin{equation}
K
=
\min\!\bigl(K_{\max},\;\max(4,\;\lfloor 2\log_2(|\mathcal{V}|+1)\rfloor)\bigr),
\label{eq:K_adaptive}
\end{equation}
which serves as a practical candidate-pool heuristic prior to metric reranking.
This logarithmic growth balances the quality of metric-aware parent selection against the increasing cost of evaluating larger candidate sets as the tree expands. A detailed ablation and discussion of this design choice are provided in Appendix~\ref{app:kablation}. The complete nearest-neighbor procedure is summarized in Algorithm~\ref{alg:metricnearest}.

\subsection{Overall RMRRT Algorithm}
\label{subsec:overall}

Algorithm~\ref{alg:rmrrt} summarizes the overall planner. At each iteration, RMRRT samples a random configuration in the ambient space, selects a node from the active tree using the metric-aware nearest-neighbor rule, and performs a local expansion using the tangent-space steering rule. If the resulting state and edge are valid, the new node is inserted into the tree, after which the opposite tree is extended toward it in a connect phase. When finite-step integration introduces noticeable equality drift, a lightweight metric-aware correction step is applied to return the configuration toward the equality-constrained manifold. This correction serves only as a local numerical recovery mechanism; admissible motion is already determined by the tangent-space geometry used in nearest-neighbor selection and steering.

\begin{lemma}[Riemannian Newton Projection]
\label{lem:riemannian_newton_projection}
Given a configuration $q$ with equality violation $\|g(q)\|>\tau$, the iterative update
\begin{equation}
q_{k+1}
=
q_k
-
G^{-1}(q_k)J(q_k)^\top
\bigl(J(q_k)G^{-1}(q_k)J(q_k)^\top\bigr)^{-1}
g(q_k)
\label{eq:riemannian_newton_projection}
\end{equation}
converges locally to the equality-constrained manifold $\mathcal{M}$, provided $G(q)$ is continuous and $J(q)$ has full row rank in a neighborhood of $\mathcal{M}$.
\end{lemma}

\begin{proof}
Each iteration solves
\begin{equation}
\Delta q^*
=
\arg\min_{\Delta q}
\frac{1}{2}\Delta q^\top G(q_k)\Delta q
\quad
\text{s.t.}
\quad
g(q_k)+J(q_k)\Delta q = 0.
\label{eq:correction_qp}
\end{equation}
The corresponding KKT system yields
\begin{equation}
\Delta q^*
=
-
G^{-1}(q_k)J(q_k)^\top
\bigl(J(q_k)G^{-1}(q_k)J(q_k)^\top\bigr)^{-1}
g(q_k),
\label{eq:deltaq_star_projection}
\end{equation}
which recovers \eqref{eq:riemannian_newton_projection}. The Schur complement remains positive definite under the same conditions used above, so the step is well-defined, and local convergence follows from standard equality-constrained Newton arguments~\cite{nocedal2006numerical}.
\end{proof}

\begin{remark}[Minimum-Norm Manifold Correction]
\label{rem:riemannian_projection}
The correction step is the minimum-$G$-norm update that removes the first-order equality violation. Because $G(q)$ already reflects nearby inequality geometry, the correction remains consistent with the same local metric used in steering and nearest-neighbor selection.
\end{remark}

Under standard regularity assumptions, the same geometric ingredients remain compatible with the incremental path-covering arguments used in probabilistic completeness analyses for constrained sampling-based planners \cite{kavraki1996probabilistic, lavalle2001randomized, hsu1997expansive, berenson2009manipulation, jaillet2012path, kingston2019exploring}. We stress that the following result establishes compatibility with these arguments rather than an independent completeness proof for the implemented planner.

\begin{proposition}[Compatibility with Standard Completeness Arguments]
\label{prop:probabilistic_completeness}
Let $\mathcal{F}\subset\mathcal{M}$ be the feasible subset of the equality-constrained manifold. Suppose there exists a feasible path
\begin{equation}
\sigma:[0,1]\rightarrow\mathcal{F}
\label{eq:feasible_path_sigma}
\end{equation}
connecting $q_{\mathrm{start}}$ to $q_{\mathrm{goal}}$, and that this path has clearance $\delta>0$ from all inequality boundaries. Assume further that the uniform sampling fraction satisfies $\rho_u>0$, the equality Jacobian $J(q)$ has full row rank for all $q\in\mathcal{M}$, and the local metric satisfies the bounded-distortion conditions stated in Proposition~\ref{prop:metric_distortion} along $\sigma([0,1])$. Then the metric-aware nearest-neighbor selection, tangent-space steering, and correction steps of RMRRT remain compatible with the standard path-covering arguments used in probabilistic completeness analyses for constrained sampling-based planning.
\end{proposition}

\begin{proof}
The argument proceeds in three steps.

\textit{Step 1 (Uniform metric bounds along the feasible path).}
Because $\sigma$ has positive clearance $\delta>0$, the barrier activation weights remain uniformly bounded along $\sigma([0,1])$. Together with Proposition~\ref{prop:metric_distortion}, this implies the existence of constants $0<\lambda\le\Lambda<\infty$ such that, for every $q\in\sigma([0,1])$ and every $v\in\mathbb{R}^n$,
\begin{equation}
\lambda \|v\|^2
\le
v^\top G(q)v
\le
\Lambda \|v\|^2.
\label{eq:G_quadratic_bounds_along_sigma}
\end{equation}

\textit{Step 2 (Uniform norm equivalence on admissible motions).}
For every tangent vector $v\in T_q\mathcal{M}$, Lemma~\ref{lem:metric_equivalence} gives
\[
v^\top M(q)v = v^\top G(q)v.
\]
Hence
\begin{equation}
\lambda \|v\|^2
\le
v^\top M(q)v
\le
\Lambda \|v\|^2,
\qquad
\forall q\in\sigma([0,1]),\;\forall v\in T_q\mathcal{M}.
\label{eq:M_quadratic_bounds_along_sigma}
\end{equation}

\textit{Step 3 (Existence of admissible local progress).}
Let $\hat v\in T_q\mathcal{M}$ denote the unit-$M(q)$-norm steering direction returned by Algorithm~\ref{alg:metricsteer}. By \eqref{eq:M_quadratic_bounds_along_sigma}, its Euclidean norm is uniformly bounded above and below, so the steering rule generates a nondegenerate tangent displacement. Now cover the feasible path $\sigma$ by a finite sequence of overlapping balls of radius $r=\eta/(2\sqrt{\Lambda})$. Because each ball has positive measure and the uniform sampling fraction satisfies $\rho_u>0$, these neighborhoods are sampled infinitely often with probability one. Whenever a tree node lies in one ball and a subsequent sample falls in the next overlapping ball, the local norm-equivalence bounds imply that the metric-aware steering rule admits progress of controlled size toward the next neighborhood. Under the additional local-neighborhood coverage assumption that the Euclidean preselection stage retains access to a node in the relevant neighborhood, the subsequent metric-aware reranking and steering preserve this local progress property. Feasibility is enforced separately through state validity, edge validity, and the correction step. Since the feasible path has positive clearance, sufficiently small steps remain inside the admissible region and can be returned to $\mathcal{M}$ with bounded correction. Therefore, under the stated regularity, clearance, and local-neighborhood coverage assumptions, the metric-aware components of RMRRT remain compatible with the standard path-covering mechanism used in probabilistic completeness arguments for constrained sampling-based planning.
\end{proof}

\begin{remark}[Theory vs. Implemented Assumptions]
\label{rem:theory_vs_implementation}
The metric construction is motivated by the exact barrier Hessian~\eqref{eq:G_ideal}, which assumes differentiable inequality margins and $G(q)\succ0$. The implemented planner instead uses the Gauss--Newton surrogate~\eqref{eq:G_GN}, which assumes only (a)~tractable proxy-point signed distances and Jacobians, (b)~the regularizer $\lambda>0$, which enforces
\begin{equation*}
G(q)\succeq\lambda I\succ0
\end{equation*}
irrespective of the surrogate quality, (c)~full row rank of $J(q)$ along the solution path, (d)~positive clearance $\delta>0$ of that path, and (e)~local-neighborhood coverage of the Euclidean preselection stage for finite $K$. The completeness-compatibility argument of Proposition~\ref{prop:probabilistic_completeness} uses only assumptions (b)--(e); the exact-Hessian form is never required.
\end{remark}

\begin{remark}[Barrier Metric as Geometric Bias]
\label{rem:probabilistic_completeness}
The barrier metric is introduced to bias exploration geometrically, not to suppress it. Equality constraints define the admissible tangent structure, while inequality information reshapes the local geometry on that space through the induced tangent-space metric. The correction step is used only to remove finite-step drift from the manifold, whereas hard feasibility is enforced separately through state and edge validity checks.
\end{remark}

\subsection{Computational Complexity and Practical Overhead}

We now examine the additional computational overhead introduced by the metric-aware operations. The dominant added cost arises from constructing the ambient barrier metric and reranking a small set of Euclidean-preselected candidates during nearest-neighbor search.

\begin{figure}[!t]
\centering
\includegraphics[width=3.0in]{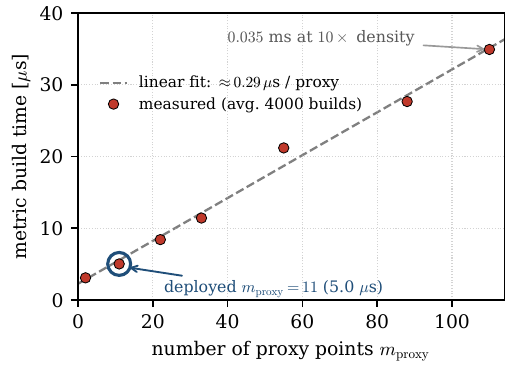}
\caption{Scaling of ambient-metric build time with proxy density in the dual-arm scenario.}
\label{fig:proxy_density}
\end{figure}

\begin{proposition}[Per-Expansion Complexity]
\label{prop:complexity}
One metric-aware expansion of RMRRT, including local metric construction, tangent projection, and nearest-neighbor reranking over $K$ Euclidean-preselected candidates, has computational complexity
\begin{equation}
O\!\left(K\,(n^2 m_{\mathrm{eq}} + n^2 m_{\mathrm{proxy}})\right),
\label{eq:complexity}
\end{equation}
where $n$ is the configuration dimension, $m_{\mathrm{eq}}$ is the number of equality constraints, $m_{\mathrm{proxy}}$ is the number of proxy points, and $K$ is the candidate size used in Algorithm~\ref{alg:metricnearest}.
\end{proposition}

\begin{proof}
We focus on the additional metric-aware cost beyond standard validity checking. For each candidate node, forming the ambient barrier metric in \eqref{eq:G_GN} requires accumulating $m_{\mathrm{proxy}}$ rank-3 outer products, with cost $O(n^2 m_{\mathrm{proxy}})$. Constructing the tangent projector requires solving for $G^{-1}(q)J(q)^\top$ and forming the associated Schur complement, whose dominant cost is $O(n^2 m_{\mathrm{eq}})$ when $m_{\mathrm{eq}} \ll n$. Since Algorithm~\ref{alg:metricnearest} reranks $K$ candidates, the total added cost is
\[
O\!\left(K\,(n^2 m_{\mathrm{eq}} + n^2 m_{\mathrm{proxy}})\right).
\]
Thus, when $m_{\mathrm{eq}} \ll n$, the metric-aware overhead is dominated by local metric construction and scales linearly with the number of proxy points.
\end{proof}

\begin{remark}[Low-Rank Structure and Efficient Inversion]
\label{rem:woodbury}
The ambient barrier metric
\[
G(q)=\lambda I+\sum_{j=1}^{m_{\mathrm{proxy}}} w_j(q)\,J_j(q)^\top J_j(q)
\]
has a low-rank-plus-diagonal structure. Let
\[
\bar J=
\big[\sqrt{w_1}J_1;\ldots;\sqrt{w_{m_{\mathrm{proxy}}}}J_{m_{\mathrm{proxy}}}\big]
\in
\mathbb{R}^{3m_{\mathrm{proxy}}\times n}.
\]
Then Woodbury’s identity~\cite{golub2013matrix} gives
\begin{equation}
G^{-1}
=
\frac{1}{\lambda}
\left(
I-\bar J^\top
(\lambda I_{3m_{\mathrm{proxy}}}+\bar J\bar J^\top)^{-1}
\bar J
\right),
\label{eq:woodbury}
\end{equation}
which reduces the cost of applying $G^{-1}$ when $m_{\mathrm{proxy}} \ll n$.
\end{remark}

\begin{table}[!t]
\centering
\caption{Proxy-density trade-off in the dual-arm scenario.}
\label{tab:proxy_density}
{
\footnotesize
\setlength{\tabcolsep}{3pt}
\begin{tabular}{@{}lcccccc@{}}
\toprule
proxy set & $m_{\mathrm{proxy}}$ & valid-ext.\ ratio & added nodes & iters. & path len. & time (s) \\
\midrule
sparse   & $2$  & $0.744$ & $1000$ & $261$ & $21.8$ & $0.766$ \\
\textbf{deployed} & $11$ & $0.752$ & $498$  & $197$ & $18.8$ & $0.435$ \\
dense    & $22$ & $0.840$ & $2066$ & $425$ & $20.8$ & $1.684$ \\
\bottomrule
\end{tabular}
}
\end{table}

\begin{remark}[Practical Overhead Relative to Validity Checking]
\label{rem:practical_overhead}
In constrained sampling-based planning, collision and validity checking often dominate runtime. The metric-aware computations reuse the same signed distances and point Jacobians already computed or cached for validity evaluation, so their marginal overhead is largely algebraic. Moreover, because the barrier metric biases expansions away from nearby invalid regions, it can partially offset its own cost by reducing rejection-driven validity-check cycles in cluttered environments.
\end{remark}
We further analyze how proxy density affects both metric-construction cost and planning behavior.
As shown in Fig.~\ref{fig:proxy_density}, the ambient-metric build time increases approximately linearly with the number of proxy points, but remains sub-millisecond even at \(10\times\) the deployed density.
Thus, within the tested range, metric construction is not the runtime bottleneck; the dominant per-iteration cost remains collision and validity checking. However, the low construction cost does not imply that using more proxy points always yields a more informative metric. Proxy density determines where the barrier metric senses obstacle proximity and therefore shapes the exploration bias. A sparse proxy set provides insufficient coverage of relevant link--obstacle proximity directions, leading to weak steering guidance and inefficient tree expansion.

In contrast, an overly dense proxy set imposes strong penalties around many local obstacle-proximity directions. This increases the valid-extension ratio, but also makes local expansion more conservative, reducing goal-directed progress and increasing the number of node expansions and iterations. This trade-off is summarized in Table~\ref{tab:proxy_density}. The deployed density, \(m_{\mathrm{proxy}}=11\), lies at the resulting efficiency knee. It achieves the lowest planning time, the fewest added nodes and iterations, and the shortest path among the tested densities, while maintaining a high valid-extension ratio. This density sweep is not used to tune the proxy density separately for each scene. Rather, it characterizes the coverage--efficiency trade-off induced by the fixed geometric link-coverage heuristic. The heuristic and the corresponding per-scenario proxy compositions are described in the experimental setup in Section~\ref{section4}. The sweep supports that this fixed choice provides a practical balance between metric informativeness and planning efficiency.

\section{Experiments}
\label{section4}

This section evaluates RMRRT under coupled equality and inequality constraints in terms of reliability, efficiency, and practical applicability. Consistent with the scope clarified in Section~\ref{section3}, the metric incorporates only distance- and proximity-based geometric inequalities, instantiated primarily from robot--obstacle proximity. Other feasibility constraints, including joint limits, self-collision, and task-specific validity conditions, are enforced separately through standard state- and edge-validity checks. Thus, the metric serves as an exploration bias, while hard feasibility is guaranteed independently by validity checking. The method is evaluated on single-arm, multi-arm, and humanoid systems across constrained manipulation scenarios of increasing geometric and kinematic complexity. We report planning success, runtime, and computational overhead, and use ablation studies to isolate the contributions of the proposed metric-aware components. Hardware experiments further demonstrate practical applicability in physical deployment. Additional videos are available on the project page: \url{https://rcilab.khu.ac.kr/rmrrt/}.

\begin{figure*}[t]
\centering
\subfloat[]{%
  \includegraphics[width=0.95\textwidth]{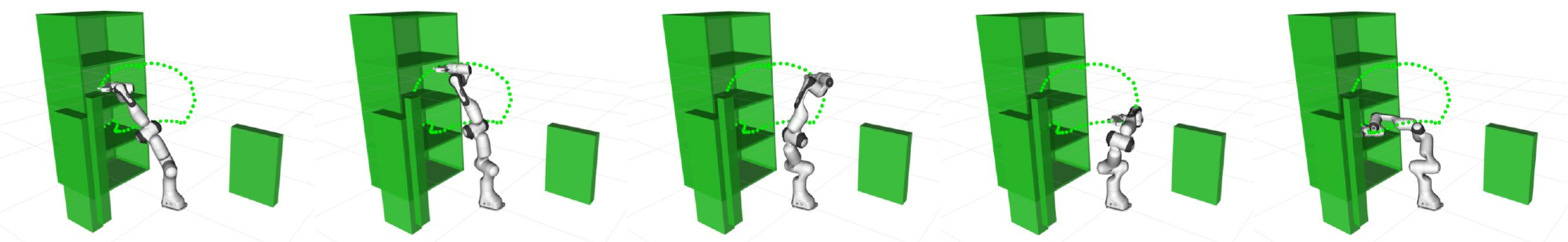}%
  \label{fig5_a}%
}\\[-0.2em]
\subfloat[]{%
  \includegraphics[width=0.95\textwidth]{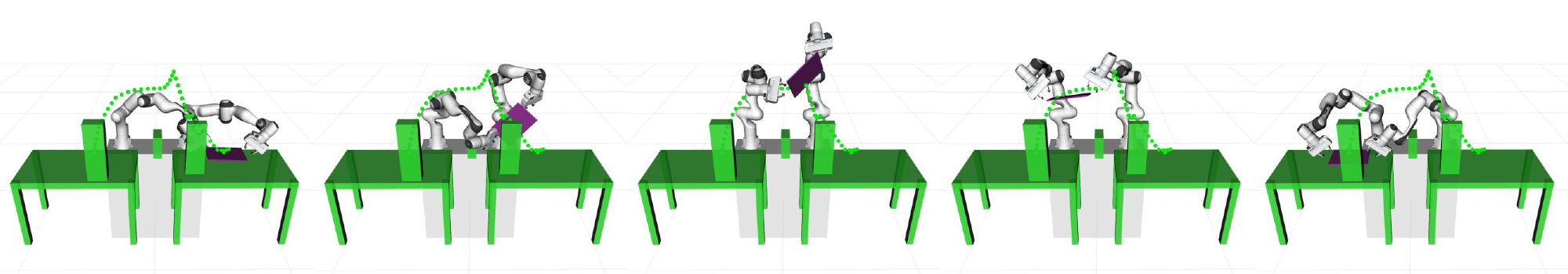}%
  \label{fig5_b}%
}\\[-0.2em]
\subfloat[]{%
  \includegraphics[width=0.95\textwidth]{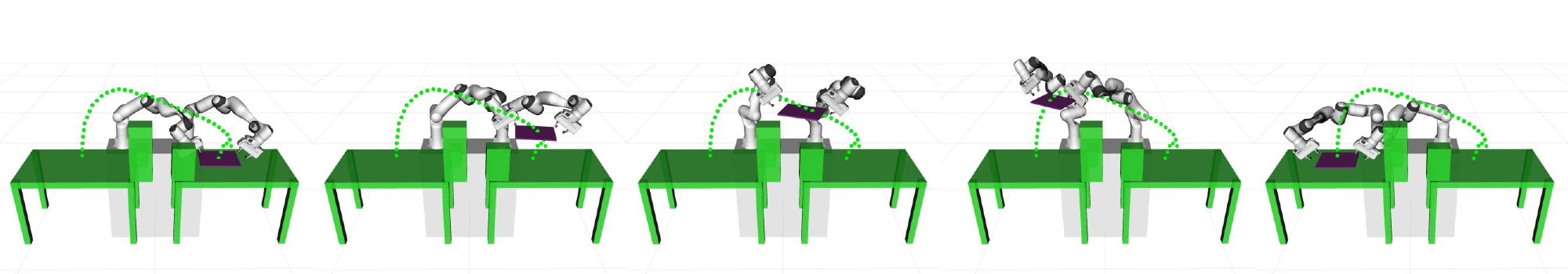}%
  \label{fig5_c}%
}\\[-0.2em]
\subfloat[]{%
  \includegraphics[width=0.95\textwidth]{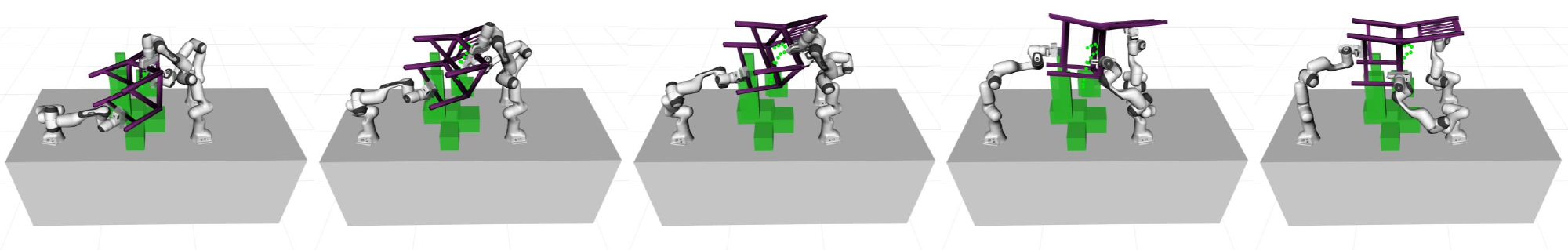}%
  \label{fig5_d}%
}
\caption{Planned trajectories in four representative manipulation scenarios. The dotted curve indicates the end-effector (or object) trajectory, and snapshots depict key configurations along each path.
(a) Single-arm cabinet reaching with fixed orientation constraint.
(b) Dual-arm constrained transfer.
(c) Dual-arm transfer with fixed orientation constraint.
(d) Triple-arm constrained manipulation.}
\label{fig:panda_paths}
\end{figure*}
\begin{table*}[t]
\centering
\caption{Planning performance across four experimental scenarios.}
\label{tab:benchmark_4scenarios}
\setlength{\tabcolsep}{5pt}
\renewcommand{\arraystretch}{1.05}

\footnotesize
\begin{tabular}{l ccc ccc}
\toprule
\multirow{2}{*}[-2pt]{\textbf{Planner}} &
\multicolumn{3}{c}{\textbf{Scenario 1}} &
\multicolumn{3}{c}{\textbf{Scenario 2}} \\
\cmidrule(lr){2-4}\cmidrule(lr){5-7}
& \textbf{Time (s)} & \textbf{Succ. (\%)} & \textbf{Path Len.}
& \textbf{Time (s)} & \textbf{Succ. (\%)} & \textbf{Path Len.} \\
\midrule

\multicolumn{7}{@{}l}{\emph{Projection- and atlas-based constrained samplers}}\\
\addlinespace[1pt]
CBiRRT \cite{berenson2009manipulation}
& $0.36 \pm 0.54$ & 100 & $8.54 \pm 3.68$
& $2.97 \pm 4.48$ & 100 & $9.07 \pm 2.75$ \\
Atlas RRT \cite{jaillet2012path}
& $1.04 \pm 2.65$ & 100 & $7.45 \pm 3.08$
& $3.87 \pm 6.34$ & 100 & $9.87 \pm 3.22$ \\
Tangent Bundle RRT \cite{kim2016tangent}
& $0.87 \pm 3.34$ & 100 & $9.12 \pm 3.05$
& $3.03 \pm 4.32$ & 100 & $9.61 \pm 3.83$ \\

\addlinespace[2pt]
\multicolumn{7}{@{}l}{
\emph{Metric-based local methods (with equality projection)}}\\
\addlinespace[1pt]
RMP \cite{ratliff2018riemannian,cheng2021rmpflow}
& $0.36 \pm 0.18$
& 56
& $5.00 \pm 1.95$
& $0.94 \pm 0.59$
& 80
& $6.57 \pm 1.59$ \\
CHOMP \cite{ratliff2009chomp}
& $2.59 \pm 2.90$
& 38
& $10.34 \pm 4.12$
& $2.62 \pm 3.73$
& 67
& $12.61 \pm 7.63$ \\
Artificial Potential Field \cite{khatib1986real}
& $0.36 \pm 0.16$
& 54
& $5.08 \pm 1.79$
& $0.93 \pm 0.54$
& 84
& $6.84 \pm 2.22$ \\

\addlinespace[2pt]
\multicolumn{7}{@{}l}{\emph{Learning-based planners}}\\
\addlinespace[1pt]
LJCMP \cite{park2024constrained}
& $1.17 \pm 3.30$ & 100 & $12.45 \pm 4.71$
& $0.28 \pm 0.65$ & 100 & $11.97 \pm 4.45$ \\
CoMPNetX \cite{qureshi2022compnetx}
& $0.27 \pm 0.67$
& 100
& $6.10 \pm 2.11$
& $3.01 \pm 3.77$
& 100
& $8.61 \pm 2.86$ \\
Precomputed Graph RRT \cite{csucan2012motion}
& $0.42 \pm 0.69$ & 100 & $8.89 \pm 3.79$
& $2.72 \pm 3.67$ & 100 & $9.06 \pm 2.62$ \\

\midrule
\textbf{RMRRT (ours)}
& $\mathbf{0.09 \pm 0.14}$ & 100 & $8.05 \pm 3.33$
& $\mathbf{0.23 \pm 0.31}$ & 100 & $10.92 \pm 4.86$ \\

\midrule\midrule
\multirow{2}{*}[-2pt]{\textbf{Planner}} &
\multicolumn{3}{c}{\textbf{Scenario 3}} &
\multicolumn{3}{c}{\textbf{Scenario 4}} \\
\cmidrule(lr){2-4}\cmidrule(lr){5-7}
& \textbf{Time (s)} & \textbf{Succ. (\%)} & \textbf{Path Len.}
& \textbf{Time (s)} & \textbf{Succ. (\%)} & \textbf{Path Len.} \\
\midrule

\multicolumn{7}{@{}l}{\emph{Projection- and atlas-based constrained samplers}}\\
\addlinespace[1pt]
CBiRRT \cite{berenson2009manipulation}
& $4.58 \pm 9.20$ & 100 & $9.53 \pm 2.78$
& $74.75 \pm 118.59$ & 90 & $10.67 \pm 2.49$ \\
Atlas RRT \cite{jaillet2012path}
& $5.71 \pm 9.27$ & 100 & $9.56 \pm 3.23$
& $42.11 \pm 105.49$ & 92 & $10.54 \pm 3.64$ \\
Tangent Bundle RRT \cite{kim2016tangent}
& $4.94 \pm 10.11$ & 100 & $9.80 \pm 7.35$
& $15.80 \pm 57.35$ & 87 & $11.01 \pm 3.89$ \\

\addlinespace[2pt]
\multicolumn{7}{@{}l}{
\emph{Metric-based local methods (with equality projection)}}\\
\addlinespace[1pt]
RMP \cite{ratliff2018riemannian,cheng2021rmpflow}
& $1.11 \pm 0.64$
& 72
& $6.76 \pm 1.23$
& $4.51 \pm 11.35$
& 88
& $9.18 \pm 1.64$ \\
CHOMP \cite{ratliff2009chomp}
& $1.55 \pm 3.83$
& 31
& $7.22 \pm 3.99$
& $0.64 \pm 0.38$
& 41
& $9.02 \pm 1.31$ \\
Artificial Potential Field \cite{khatib1986real}
& $1.06 \pm 0.66$
& 73
& $6.86 \pm 1.76$
& $2.23 \pm 1.33$
& 85
& $9.07 \pm 1.51$ \\

\addlinespace[2pt]
\multicolumn{7}{@{}l}{\emph{Learning-based planners}}\\
\addlinespace[1pt]
LJCMP \cite{park2024constrained}
& $0.48 \pm 1.56$ & 100 & $12.31 \pm 3.34$
& $0.55 \pm 1.69$ & 100 & $16.44 \pm 5.65$ \\
CoMPNetX \cite{qureshi2022compnetx}
& $4.44 \pm 6.29$
& 100
& $8.91 \pm 2.86$
& $19.25 \pm 52.08$
& 100
& $9.80 \pm 2.44$ \\
Precomputed Graph RRT \cite{csucan2012motion}
& $2.07 \pm 3.29$ & 100 & $9.51 \pm 2.80$
& $13.34 \pm 35.11$ & 100 & $11.41 \pm 4.23$ \\

\midrule
\textbf{RMRRT (ours)}
& $\mathbf{0.25 \pm 0.26}$ & 100 & $12.63 \pm 8.17$
& $\mathbf{0.31 \pm 0.39}$ & 100 & $11.09 \pm 3.42$ \\
\bottomrule
\end{tabular}
\end{table*}

\subsection{Simulation Benchmarks}
\label{simulation_benchmark}
The simulation benchmarks comprise constrained manipulation scenarios ranging from single-arm to triple-arm systems. All experiments were conducted on the same computing platform equipped with an Intel Core i7-13700F processor, 32~GB of RAM, and an NVIDIA GeForce RTX 4060 GPU. For each scenario, we evaluated 100 independent feasible queries, each with independently generated start and goal configurations and obstacle layouts on the constraint manifold. A trial was considered successful only if a valid path was found within 300~s. All methods were evaluated under identical experimental conditions and validity-checking procedures, while their internal parameters were tuned as fairly as possible according to quantities with analogous functional roles.

The proxy points used by the barrier metric were fixed once, offline, using a link-coverage heuristic. End-effector proxies were placed at the tips of the kinematic chains, intermediate-link proxies were distributed approximately uniformly between each base and tip, and an object-coupling proxy was added at the centroid of the end-effector positions for multi-arm transport tasks. A single point was used per host link in the manipulator scenarios, whereas larger links and the carried object of the Unitree H1-2 were represented by local point clusters and corner proxies, respectively. The per-scenario proxy composition and problem scale are reported in Appendix~\ref{app:params} and Table~\ref{tab:proxy_params}. The comparison includes three complementary classes of baselines. First, we consider projection- and atlas-based constrained planners, including CBiRRT~\cite{berenson2009manipulation}, Atlas-RRT~\cite{jaillet2012path}, and Tangent Bundle RRT~\cite{kim2016tangent,jaillet2012path}, which operate directly on implicitly defined equality-constrained manifolds. Second, to evaluate the role of inequality-aware local geometry, we include Artificial Potential Field (APF)~\cite{khatib1986real}, CHOMP~\cite{ratliff2009chomp}, and Riemannian Motion Policy (RMP)~\cite{ratliff2018riemannian,cheng2021rmpflow}, each augmented with the same equality-projection and state- and edge-validity procedures. Third, we compare with learning- and precomputation-based planners, namely LJCMP~\cite{park2024constrained}, CoMPNetX~\cite{qureshi2022compnetx}, and Precomputed Graph RRT~\cite{csucan2012motion}. Fig.~\ref{fig:panda_paths} shows representative tasks and planned trajectories, and Table~\ref{tab:benchmark_4scenarios} summarizes performance across four scenarios of increasing geometric and kinematic complexity.

RMRRT achieves a $100\%$ success rate in all scenarios while maintaining low mean planning time and low inter-trial variability. The advantage is particularly pronounced in Scenario~4, where a 21-DoF triple-arm system must traverse a narrow admissible region under 12 equality constraints. Conventional manifold-based planners maintain equality feasibility through tangent-space expansion and projection-based correction, but their nearest-neighbor and steering rules generally do not account directly for nearby inequality boundaries. Consequently, candidate motions may be poorly aligned with locally admissible directions and subsequently rejected, leading to repeated steering--rejection cycles in cluttered environments. The APF-, CHOMP-, and RMP-based constrained variants sequentially combine inequality-aware ambient guidance with an independently defined equality projection. Because this projection restores equality feasibility rather than preserving the geometry of the original update, the corrected tangent motion may lose part of the intended obstacle-avoidance bias. RMRRT instead derives the ambient metric $G(q)$, the $G$-orthogonal projector, and the induced tangent-space metric $M(q)$ from a common geometric construction. Inequality-aware steering and equality-consistent motion generation are therefore defined consistently under the same metric, which also governs nearest-neighbor selection. Appendix~\ref{app:ambient} further isolates the complementary roles of barrier-informed local guidance and sampling-based nonlocal exploration by removing all equality constraints. The learning- and precomputation-based baselines achieve high success rates, but their planning time, runtime variability, or path length generally degrade as the problem dimension and constraint complexity increase. In contrast, RMRRT maintains high reliability and low online planning time without relying on task-specific training data or precomputed roadmaps. Overall, the results demonstrate the benefit of combining an inequality-aware ambient metric, equality-consistent tangent projection, metric-aware node selection, and sampling-based nonlocal exploration within a unified geometric framework.

\begin{table*}[t]
\centering
\caption{Computational cost in Scenario~4. (a) per-iteration runtime during planning. (b) one-time offline precomputation required before planning, and whether the resulting artifact survives a change in the planning problem.}
\label{tab:overhead_breakdown}
\footnotesize
\setlength{\tabcolsep}{6pt}
\renewcommand{\arraystretch}{1.1}
\begin{tabular}{lccc}
\toprule
\multicolumn{4}{@{}l}{\emph{(a) Per-iteration runtime (ms)}} \\
\midrule
\textbf{Component} & \textbf{RMRRT} & \textbf{CBiRRT} & \textbf{LJCMP} \\
\midrule
Metric build          & 0.31 & N/A & N/A \\
Nearest-neighbor      & 0.01 & 0.04 & 1.65 \\
Steering              & 17.56 & 999.35 & 1.48 \\
Collision / validity  & 0.95 & 4.83 & 1.72 \\
\midrule
Total per iteration   & 18.83 & 1004.22 & 4.86 \\
\bottomrule
\end{tabular}

\vspace{3mm}

\begin{tabular}{lll}
\toprule
\multicolumn{3}{@{}l}{\emph{(b) One-time offline precomputation}} \\
\midrule
\textbf{Method} & \textbf{Offline compute} & \textbf{Reusable after problem change} \\
\midrule
Precomputed Graph RRT & $\sim$$90$\,s ($10{,}000$-config roadmap) & Rebuild if environment/scenario changes \\
LJCMP & $\sim$$263$\,s ($90$\,s data $+$ $173$\,s training) & Retrain if robot/constraint changes \\
\textbf{RMRRT (ours)} & \textbf{none (fully online)} & \textbf{No retraining needed} \\
\bottomrule
\end{tabular}
\end{table*}

\subsection{Computational Overhead}
\label{computational_overhead}
To complement the theoretical complexity analysis in
Section~\ref{section3}, Table~\ref{tab:overhead_breakdown} jointly
examines the online computational cost incurred during planning and the
one-time offline preparation required by representative baselines. The
online analysis is conducted on Scenario~4, whose ambient
configuration-space dimension is $n=21$, and separates the costs of
metric construction, nearest-neighbor evaluation, steering, and
collision/validity checking. The offline analysis reports the
wall-clock time required to construct the learned or precomputed
representation used by each method and whether the resulting artifact
remains reusable after changes to the planning problem. Together, these
analyses characterize both the runtime overhead introduced by the
proposed metric-aware components and the practical distinction between
a fully online planner and methods that depend on problem-specific
offline representations.

As shown in Table~\ref{tab:overhead_breakdown}(a), the additional
per-iteration cost of metric construction in RMRRT is modest and
accounts for only a small fraction of the total runtime. Metric-aware
nearest-neighbor selection likewise introduces limited overhead because
a compact candidate set is first formed using a Euclidean criterion,
after which metric-based reranking is applied only to those candidates. The largest computational difference between RMRRT and CBiRRT appears
in the steering stage. CBiRRT uses projection-based local extension:
it first generates an ambient Euclidean step and then projects the
resulting configuration onto the equality-constrained manifold. Because
this projection-based expansion is repeatedly invoked during local
steering, it substantially increases the per-iteration computational
cost. RMRRT instead first projects the target displacement onto the
tangent space, generates a metric-induced steering direction within that
space, and applies only a lightweight local correction to remove the
small equality drift introduced by the finite extension step. Thus, the
equality-manifold geometry enters the generation of the extension
direction itself rather than only through post-extension projection.
This structural difference reduces the steering cost from
$999.35$\,ms for CBiRRT to $17.56$\,ms for RMRRT, and the total
per-iteration runtime from $1004.22$\,ms to $18.83$\,ms. Relative to LJCMP, RMRRT incurs a higher per-iteration cost because it
constructs and uses local geometric information online. Specifically,
the per-iteration cost is $18.83$\,ms for RMRRT and $4.86$\,ms for
LJCMP. However, this difference does not directly determine the total
planning time. In Scenario~4, LJCMP required approximately $77.0 \pm 270$ iterations on average to obtain a solution, whereas RMRRT required
approximately $13.6 \pm 17$. This indicates that the additional geometric
computation in RMRRT produces more effective progress per expansion and
substantially reduces the number of expansion cycles required to reach
a feasible solution. Its higher per-iteration cost is therefore offset
by the smaller number of iterations.

Table~\ref{tab:overhead_breakdown}(b) reveals a separate practical
distinction that is not captured by per-iteration timing alone.
Precomputed Graph RRT and LJCMP require a problem-specific roadmap or
learned latent representation to be constructed offline, and these
artifacts generally need to be rebuilt or retrained when the robot,
environment, or constraint specification changes. In contrast, RMRRT
constructs the required local geometry online from quantities already
available during planning, including the constraint Jacobian,
proxy-point Jacobians, and proximity information. The same planner can
therefore be applied directly to a new robot, constraint specification,
or environment without maintaining a separate offline representation. This distinction should be interpreted separately from the quantitative
planning-speed comparison. Table~\ref{tab:overhead_breakdown}(a)
evaluates the online planning efficiency of RMRRT, whereas the offline
analysis in Table~\ref{tab:overhead_breakdown}(b) characterizes
adaptability and ease of deployment under changes to the planning
problem. Overall, RMRRT introduces additional online geometric
computation relative to standard constrained planning, but this overhead
remains controlled and is compensated by more informative local
exploration and fewer expansion cycles. At the same time, the method
avoids problem-specific offline learning or precomputation, providing a
favorable trade-off between online planning efficiency and adaptability
to changes in the robot, constraints, and environment.

\begin{figure*}[!t]
\centering

\subfloat[]{%
  \includegraphics[width=0.95\textwidth]{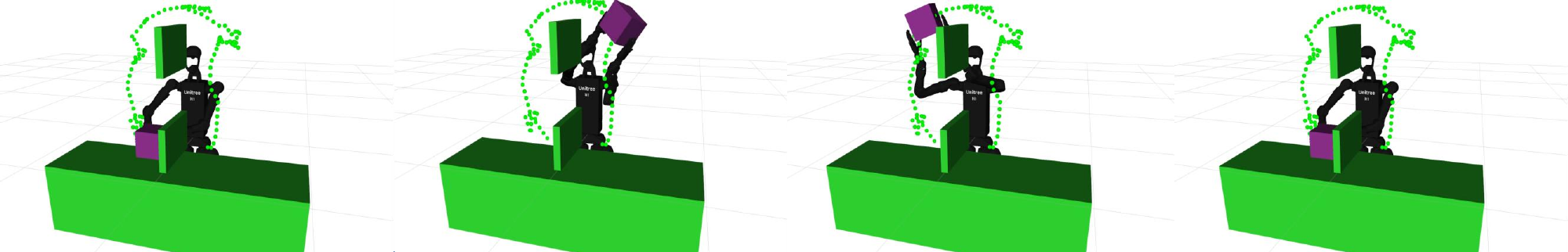}%
  \label{fig6_a}%
}

\subfloat[]{%
  \includegraphics[width=0.95\textwidth]{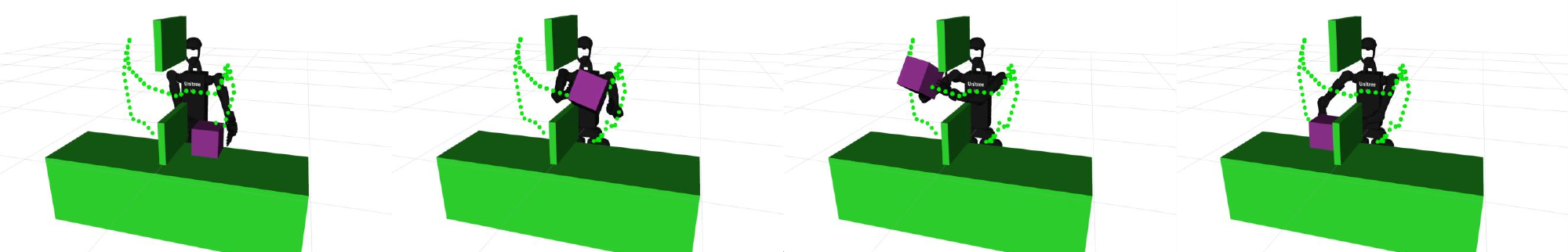}%
  \label{fig6_b}%
}
\caption{
Ablation study on the narrow-passage constrained transfer task in Scenario 5 using a Unitree H1-2 humanoid. (a) No barrier variant. (b) Full RMRRT. The no-barrier variant exhibits a longer detour around the restrictive region, whereas Full RMRRT generates a more direct trajectory through the passage.
}
\label{fig:h1_2_barrier_paths}
\vspace{-0.6em} %
\end{figure*}
\begin{table*}[t]
\centering
\caption{Ablation study on the contribution of the metric-aware components (Scenario 5).}
\label{tab:ablation}
\renewcommand{\arraystretch}{1.2}
\footnotesize
\begin{tabular}{lcccccc}
\toprule
 & \multicolumn{2}{c}{\textbf{Outcome}} & \multicolumn{1}{c}{\textbf{Exploration Quality}} & \multicolumn{3}{c}{\textbf{Cost Breakdown (\(\mu\)s/iter)}} \\
\cmidrule(lr){2-3} \cmidrule(lr){4-4} \cmidrule(lr){5-7}
\textbf{Method} & Time (s) & Path Len. (m) & Valid Ext. Ratio & Check & Steer & Near \\
\midrule
No barrier             & $21.26 \pm 7.34$  & $35.39 \pm 9.84$  & $0.333$          & $12349$          & $1193$ & $5.4$ \\
Steering-only barrier  & $17.46 \pm 7.42$  & $28.13 \pm 7.38$  & $\mathbf{0.349}$ & $\mathbf{8174}$  & $1442$ & $3.4$ \\
Nearest-only barrier   & $18.75 \pm 8.97$  & $25.66 \pm 6.36$  & $0.329$          & $10987$          & $1167$ & $5.1$ \\
\textbf{Full RMRRT}    & $19.54 \pm 11.08$ & $\mathbf{25.21 \pm 2.82}$ & $\mathbf{0.349}$ & $9243$           & $1541$ & $7.3$ \\
\bottomrule
\end{tabular}
\end{table*}

\subsection{Ablation Study}
\label{subsec:ablation_barrier}

To isolate the contribution of the proposed metric-aware components, we
conduct an ablation study on the narrow-passage constrained transfer
task in Scenario~5 using a Unitree H1-2 single-humanoid dual-arm system.
In this setting, densely arranged obstacles strongly restrict the
locally admissible directions and passable regions on the
equality-constrained manifold, making the scenario well suited for
evaluating how each metric-aware component affects exploration.

We compare the following four variants:
\begin{itemize}[leftmargin=1.5em]
    \item \textbf{Full RMRRT}: metric-aware nearest-neighbor selection
    and metric-induced steering.
    
    \item \textbf{Steering-only barrier}: the barrier-informed metric is
    used only for steering, while nearest-neighbor selection remains
    Euclidean.
    
    \item \textbf{Nearest-only barrier}: the barrier-informed metric is
    used only for nearest-neighbor selection, while standard
    tangent-space steering is retained.
    
    \item \textbf{No barrier}: neither nearest-neighbor selection nor
    steering uses the barrier-informed metric.
\end{itemize}

All other components, including the sampling strategy, validity
checking, termination conditions, and major hyperparameters, are kept
identical across the four variants. Table~\ref{tab:ablation} reports
planning time and path length together with the valid extension ratio
and per-iteration computational cost. This decomposition allows us to
distinguish whether the observed gains arise from steering,
nearest-neighbor selection, or their consistent coupling.

As shown in Fig.~\ref{fig:h1_2_barrier_paths} and
Table~\ref{tab:ablation}, applying the barrier-informed metric only to
steering primarily improves local extension quality. The
Steering-only barrier variant achieves the highest valid extension
ratio among the four variants, tied with Full RMRRT, and substantially reduces the
per-iteration validity-checking cost relative to the No barrier
baseline. Both variants that use metric-induced steering reach the same $0.349$ ratio, whereas the two variants without it remain lower, isolating metric-induced steering as the component that improves local extension quality. This indicates that metric-induced steering suppresses
locally unproductive motions toward nearby inequality boundaries before
they develop into repeated steering--rejection cycles. Consequently, the
additional geometric computation required for steering is partially
offset by the reduction in rejection-driven validity-checking overhead.

By contrast, applying the barrier-informed metric only to
nearest-neighbor selection primarily improves parent-node quality and
path formation. The Nearest-only barrier variant significantly reduces
the average path length relative to the No barrier baseline, indicating
that metric-aware parent selection favors nodes with better local
connectability under the induced geometry. However, its valid extension
ratio remains comparable to that of the No barrier variant. Thus,
improved parent selection alone is insufficient to improve local
extension quality when the steering rule itself remains unaware of the
inequality-induced geometry. Even when the planner starts from a more
favorable parent node, the resulting extension may still proceed in a
locally unproductive direction.

Full RMRRT combines both components under the same induced tangent-space
metric. It achieves the shortest average path length and the smallest
path-length standard deviation among the four variants, demonstrating
improvements in both mean trajectory quality and inter-trial
consistency. Overall, metric-induced steering directly improves local
extension quality, whereas metric-aware nearest-neighbor selection
improves parent-node quality and path formation. Each component provides
a meaningful benefit individually, but their consistent use under a
shared local geometry aligns parent selection and local expansion under
the same geometric criterion, yielding the most reliable planning
performance.
\begin{figure*}[!t]
\centering
\includegraphics[width=\textwidth]{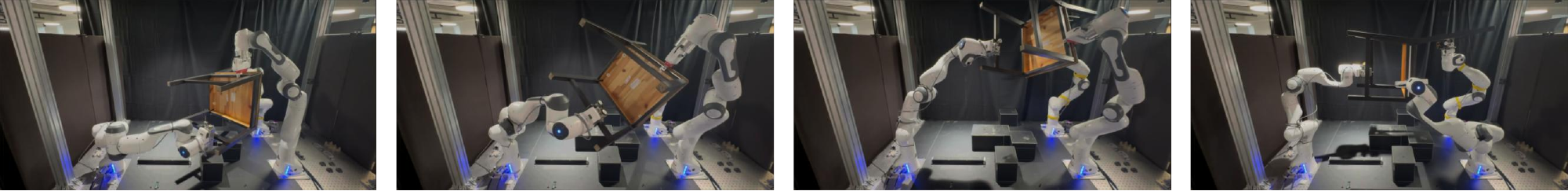}
\caption{Real-world validation of RMRRT on a triple-arm constrained manipulation task. The robots cooperatively manipulate a table under simultaneous equality and inequality constraints, including grasp-maintenance and collision-avoidance requirements.}
\label{fig:real_triple_arm}
\end{figure*}
\begin{table*}[t]
\centering
\caption{Scenario~4 deployment evaluation: robustness to signed-distance noise and equality-constraint residuals.}
\label{tab:realworld_quant}
\footnotesize
\setlength{\tabcolsep}{3pt}
\renewcommand{\arraystretch}{1.15}
\begin{minipage}[t]{0.47\textwidth}
\begin{tabular*}{\linewidth}[t]{@{\extracolsep{\fill}}lcc@{}}
\toprule
\multicolumn{3}{@{}l}{\emph{Robustness to hardware noise}} \\
\midrule
\textbf{$\sigma$ (cm)} & \textbf{Success (\%)} & \textbf{Time (s)} \\
\midrule
$0.0$ & $100$ & $0.26 \pm 0.36$ \\
$0.5$ & $100$ & $0.32 \pm 0.54$ \\
$1.0$ & $100$ & $0.51 \pm 1.38$ \\
$1.5$ & $100$ & $0.35 \pm 0.63$ \\
$2.0$ & $100$ & $0.55 \pm 1.42$ \\
\bottomrule
\end{tabular*}
\end{minipage}\hfill
\begin{minipage}[t]{0.47\textwidth}
\centering
\begin{tabular*}{\linewidth}[t]{@{\extracolsep{\fill}}lcc@{}}
\toprule
\multicolumn{3}{@{}l}{\emph{Equality-constraint residual $\|g(q)\|$}} \\
\midrule
\textbf{Planner} & \textbf{Mean $\|g(q)\|$} & \textbf{Max $\|g(q)\|$} \\
\midrule
CBiRRT                & $1.4\times10^{-3}$ & $5.0\times10^{-3}$ \\
Tangent Bundle RRT    & $1.8\times10^{-3}$ & $4.9\times10^{-3}$ \\
Precomputed Graph RRT & $1.3\times10^{-3}$ & $4.9\times10^{-3}$ \\
LJCMP                 & $9.0\times10^{-4}$ & $5.0\times10^{-3}$ \\
\textbf{RMRRT (ours)} & $\mathbf{7.1\times10^{-4}}$ & $\mathbf{4.3\times10^{-3}}$ \\
\bottomrule
\end{tabular*}
\end{minipage}
\end{table*}

\subsection{Real-World Quantitative Evaluation}
\label{subsec:real_world_quantitative}
To evaluate the practical applicability of the proposed method beyond simulation, we validated RMRRT on the triple-arm cooperative manipulation task in Scenario~4. As shown in Fig.~\ref{fig:real_triple_arm}, the three manipulators transport a jointly grasped object while continuously satisfying a closed-chain equality constraint throughout the motion. Accordingly, this experiment evaluates both robustness to distance-estimation uncertainty arising in real-world deployment and the consistency of the inter-robot kinematic relationships required for coordinated physical execution. Specifically, we report the execution success rate and planning time under increasing zero-mean Gaussian noise added to the sensor-measured signed distances, which stresses the planner's robustness to distance-estimation uncertainty encountered in real-world deployment, together with the equality-constraint residuals along the generated paths. The quantitative results are summarized in Table~\ref{tab:realworld_quant}, while additional qualitative demonstrations and execution videos are provided on the project page.

To evaluate robustness to distance-estimation errors, zero-mean Gaussian noise with increasing standard deviation was added to the sensor-measured signed-distance inputs used to construct the metric during hardware execution. As shown in Table~\ref{tab:realworld_quant}, RMRRT maintained its execution success rate across all tested noise levels, while the planning time increased only moderately. This indicates that distance-estimation errors reduce the informativeness of the exploration bias, but do not compromise the feasibility-preserving structure of the proposed method. Table~\ref{tab:realworld_quant} also reports the equality-constraint residuals along the generated paths, including interpolated intermediate states. All methods maintained equality feasibility within the task tolerance, whereas RMRRT achieved the lowest mean and maximum residuals. This result suggests that guiding exploration away from active inequality boundaries while respecting the tangent geometry of the equality manifold requires smaller manifold corrections and yields more numerically consistent projections. Consequently, RMRRT remains robust to distance-estimation uncertainty while preserving the shared-grasp consistency required for coordinated execution.

\begin{table*}[t]
\centering
\caption{Per-scenario problem scale and proxy composition.}
\label{tab:proxy_params}
\setlength{\tabcolsep}{5pt}
\renewcommand{\arraystretch}{1.1}

\begin{tabular}{lccccccccc}
\toprule
\textbf{Scenario} & \textbf{$n$} & \textbf{$m_{\mathrm{eq}}$} & \textbf{Obs.} & \textbf{EE} & \textbf{Int.} & \textbf{Obj.} & \textbf{Cl.} & \textbf{Cor.} & $\boldsymbol{m_{\mathrm{proxy}}}$ \\
\midrule
1: Single arm    & 7  & 2  & 4 & 1 & 4  & 1 & 1 & 0 & 6  \\
2: Dual arm      & 14 & 6  & 5 & 2 & 8  & 1 & 1 & 0 & 11 \\
3: Dual arm      & 14 & 8  & 5 & 2 & 8  & 1 & 1 & 0 & 11 \\
4: Triple arm    & 21 & 12 & 7 & 3 & 12 & 1 & 1 & 0 & 16 \\
5: H1-2 humanoid & 14 & 6  & 4 & 2 & 4  & 1 & 7 & 8 & 57 \\
\bottomrule
\end{tabular}

\vspace{2pt}
\footnotesize
\emph{Note:} \(n\), \(m_{\mathrm{eq}}\), and Obs. denote the configuration-space dimension, equality-constraint dimension, and obstacle count, respectively. EE/Int./Obj. denote end-effector, intermediate-link, and object-coupling host points, while Cl. and Cor. denote local-cluster size and carried-object corners; \(m_{\mathrm{proxy}}=(\mathrm{EE}+\mathrm{Int}+\mathrm{Obj})\times\mathrm{Cl}+\mathrm{Cor}\).
\end{table*}
\section{Conclusion}
\label{section5}

This paper presented RMRRT, a sampling-based motion planning framework
for inequality-aware exploration on equality-constrained manifolds. By
constructing an ambient barrier metric and its induced tangent-space
metric, RMRRT incorporates inequality-induced geometry directly into
steering and nearest-neighbor selection. This biases tree expansion away
from nearby infeasible regions while preserving first-order equality
consistency, rather than relying solely on post-hoc collision rejection.
Experiments across single-arm, multi-arm, and humanoid manipulation
scenarios demonstrated high planning reliability and improved planning
efficiency over representative constrained planning baselines. The
practical applicability of the method was further validated through
real-world hardware experiments. Additional demonstrations and source
code are available on the project page:
\url{https://rcilab.khu.ac.kr/rmrrt/}.

The current implementation instantiates the metric primarily from
signed-distance-based geometric proxy inequalities. This reflects a
practical limitation and design trade-off: because $G(q)$ is evaluated
repeatedly during planning, each metric contribution must remain
informative, inexpensive, and numerically stable. Future work will
investigate learned SDF/RDF models and constraint surrogates that
efficiently approximate inequality margins, gradients, or
positive-semidefinite metric contributions. Such models could extend the
range of inequality information incorporated into RMRRT while preserving
its role as an efficient local geometry module.

Another direction is to extend the proposed tangent-space geometry beyond
the current bidirectional RRT implementation. Because the induced metric
defines local distance evaluation, nearest-neighbor selection, and
steering independently of the search backbone, it could be incorporated
into broader manifold sampling-based frameworks such as IMACS~\cite{kingston2019exploring}, including metric-aware
variants of PRM, EST, or unidirectional RRT. It may also provide
geometry-aware feasible initializations for constrained trajectory
optimization. In particular, the ambient-space comparison with
\textsc{cuRobo} suggests a complementary relationship between reliable
nonlocal exploration through sampling and efficient optimization-based
trajectory refinement.
\begin{figure}[t]
\centering
\subfloat[]{
    \includegraphics[width=1.55in]{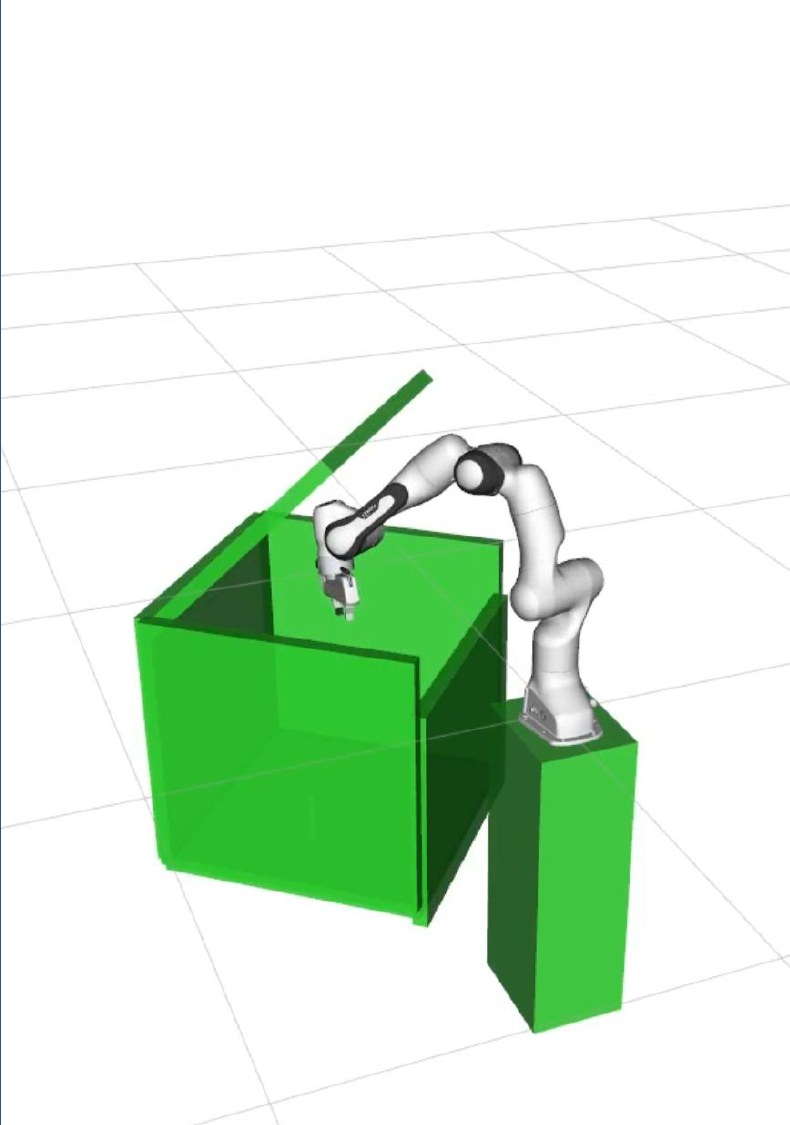}
    \label{fig:mbm_box}
}
\hfil
\subfloat[]{
    \includegraphics[width=1.55in]{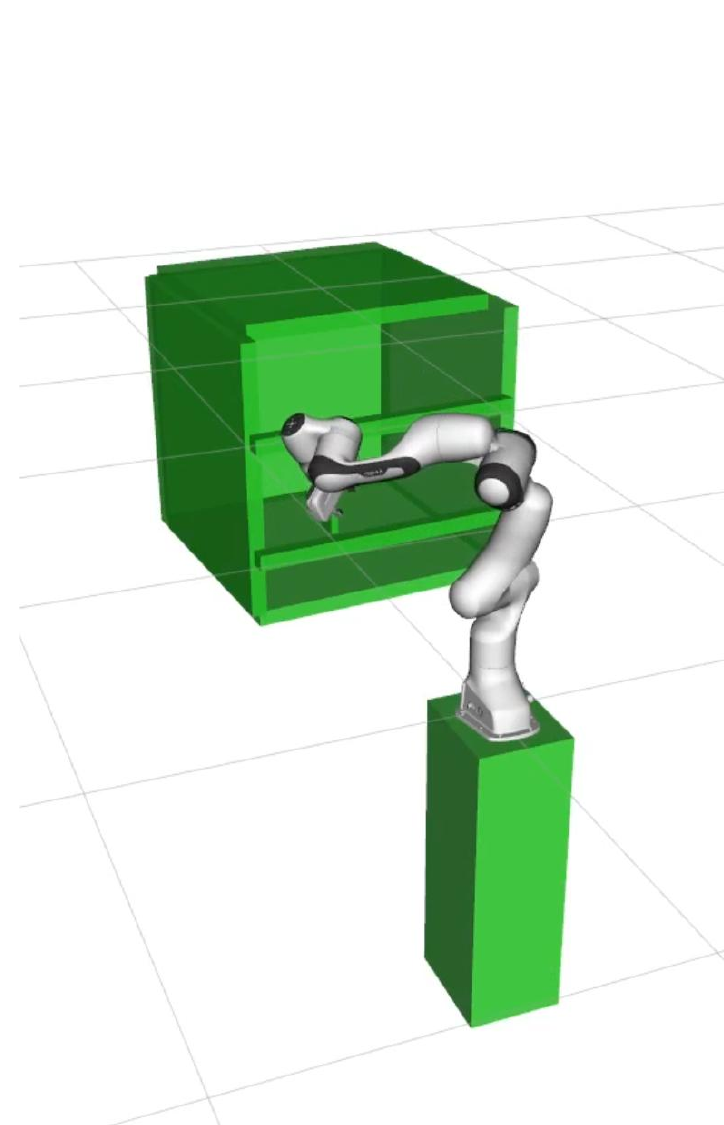}
    \label{fig:mbm_cage}
}
\caption{
MotionBenchMaker environments used for the ambient-space evaluation with a
7-DoF Franka Panda. Obstacles are shown in green:
(a) the relatively open \texttt{box} scene and
(b) the more confined \texttt{cage} scene, which contains a narrow interior passage.}
\label{fig:mbm_scenes}
\end{figure}

\begin{table}[t]
\centering

\caption{Ambient-space planning performance on MotionBenchMaker without equality constraints.}
\label{tab:ambient_mbm}
\setlength{\tabcolsep}{4pt}
\renewcommand{\arraystretch}{1.12}

\footnotesize
\begin{tabular}{llccc}
\toprule
\textbf{Scene} & \textbf{Planner} & \textbf{Succ.\,(\%)} & \textbf{Time\,(s)} & \textbf{Path\,(rad)} \\
\midrule
\multirow{5}{*}{Box}
 & cuRobo~\cite{sundaralingam2023curobo}   & 100 & $0.062 \pm 0.010$ & $3.64 \pm 0.28$ \\
 & APF~\cite{khatib1986real}               & 23 & $3.22 \pm 2.08$ & $3.77 \pm 0.41$ \\
 & CHOMP~\cite{ratliff2009chomp}           & 78 & $0.42 \pm 0.17$ & $3.59 \pm 0.15$ \\
 & RMP~\cite{ratliff2018riemannian,cheng2021rmpflow}
   & 42 & $2.58 \pm 1.22$ & $3.60 \pm 0.15$ \\
 & \textbf{RMRRT (ours)}
   & \textbf{100} & $0.017 \pm 0.015$ & $4.33 \pm 0.57$ \\
\midrule
\multirow{5}{*}{Cage}
 & cuRobo~\cite{sundaralingam2023curobo}   & 100 & $0.069 \pm 0.031$ & $4.55 \pm 0.21$ \\
 & APF~\cite{khatib1986real}               & 53 & $4.98 \pm 1.87$ & $5.15 \pm 0.36$ \\
 & CHOMP~\cite{ratliff2009chomp}           & 11 & $0.81 \pm 0.42$ & $5.38 \pm 0.64$ \\
 & RMP~\cite{ratliff2018riemannian,cheng2021rmpflow}
   & 39 & $5.66 \pm 1.97$ & $5.29 \pm 0.56$ \\
 & \textbf{RMRRT (ours)}
   & \textbf{100} & $0.081 \pm 0.103$ & $6.44 \pm 1.60$ \\
\midrule
\multirow{5}{*}{Overall}
 & cuRobo~\cite{sundaralingam2023curobo}   & 100 & $0.065 \pm 0.023$ & $4.09 \pm 0.52$ \\
 & APF~\cite{khatib1986real}               & 38 & $4.43 \pm 2.10$ & $4.73 \pm 0.74$ \\
 & CHOMP~\cite{ratliff2009chomp}           & 44.5 & $0.45 \pm 0.22$ & $3.69 \pm 0.47$ \\
 & RMP~\cite{ratliff2018riemannian,cheng2021rmpflow}
   & 40.5 & $4.72 \pm 2.27$ & $4.78 \pm 0.91$ \\
 & \textbf{RMRRT (ours)}
   & \textbf{100} & $0.049 \pm 0.081$ & $5.38 \pm 1.60$ \\
\bottomrule
\end{tabular}
\end{table}
More broadly, the proposed metric may support extensions to
asymptotically optimal planners, although preserving their guarantees
will require further analysis of metric-dependent neighborhoods,
connection radii, and rewiring rules. A complementary direction is to
combine the local inequality-aware geometry of RMRRT with global
geometric representations of constrained configuration spaces, such as
those used for closed-chain and singularity-aware planning. Such a
combination could jointly exploit global structural information and
local avoidance of active inequality boundaries.

\appendices
\section{Proxy and Metric Parameters}
\label{app:params}
\noindent To reproduce the proxy setup, Table~\ref{tab:proxy_params} reports the configuration-space dimension $n$, equality-constraint dimension $m_{\mathrm{eq}}$, obstacle count, proxy composition, and number of proxy points $m_{\mathrm{proxy}}$ for each scenario. Proxy placement follows predefined rules based solely on the geometry of the robot links and the carried object. Obstacle geometry enters collision-margin evaluation during planning, and the same scene information is available to the baseline planners. Scenarios~1--4 use the parameter set
\[
\begin{gathered}
\delta_{\mathrm{act}}=0.15\,\mathrm{m},\qquad
\delta_{\mathrm{safe}}=0.02\,\mathrm{m},\\
\alpha_{\mathrm{ee}}=\alpha_{\mathrm{int}}=\alpha_{\mathrm{obj}}=1.0,
\qquad \psi_{\max}=10^4.
\end{gathered}
\]
Scenario~5 uses a milder barrier with
\[
\begin{gathered}
\delta_{\mathrm{act}}=0.03\,\mathrm{m},\qquad
\delta_{\mathrm{safe}}=0.01\,\mathrm{m},\\
\alpha_{\mathrm{ee}}=\alpha_{\mathrm{int}}=\alpha_{\mathrm{obj}}=0.05,
\qquad \psi_{\max}=50,
\end{gathered}
\]
to avoid overly conservative tree expansion in the H1-2 humanoid narrow-passage task. The regularization parameter $\lambda=1.0$ and kernel offset $\epsilon_d=10^{-3}$ are shared across all scenarios. For the H1-2 humanoid, $\mathrm{Cl.}=7$ denotes a cluster with one host point and six offsets along the local coordinate axes. The proxy representation provides explicit collision margins and their derivatives for deriving and evaluating inequality-informed Riemannian steering on the equality-constraint manifold. The same metric construction applies to other inequality representations that satisfy the required regularity conditions. The public implementation computes constraint margins from the robot and environment models used for collision checking and state-validity evaluation. It incorporates self-collision, joint-limit, and obstacle-avoidance constraints into the metric using the parameter set
\[
\begin{gathered}
\delta_{\mathrm{obs}}=0.15\,\mathrm{m},\quad
\delta_{\mathrm{self}}=0.05\,\mathrm{m},\quad
\delta_{\mathrm{joint}}=0.2\,\mathrm{rad},\\
\alpha_c=\delta_c^2,\quad
\delta_{\mathrm{safe}}=0,\quad
\lambda=1,\quad
\epsilon_d=10^{-3},\\
\psi_{\max}=10^4,\quad
\beta=0.5,\quad
\eta=0.2.
\end{gathered}
\]

\section{Ambient-Space Comparison on MotionBenchMaker}
\label{app:ambient}
\noindent This appendix isolates the inequality-aware component of RMRRT by
removing all equality constraints. We evaluate a 7-DoF Franka Panda
on the \texttt{box} and \texttt{cage} scenes of MotionBenchMaker
\cite{chamzas2022motionbenchmaker}, as shown in
Fig.~\ref{fig:mbm_scenes}. RMRRT is compared with
APF~\cite{khatib1986real}, CHOMP~\cite{ratliff2009chomp},
RMP~\cite{ratliff2018riemannian,cheng2021rmpflow}, and the
GPU-parallel trajectory optimizer \textsc{cuRobo}
\cite{sundaralingam2023curobo}. All methods are evaluated in the same
inequality-only collision environment, and success rate, planning
time, and configuration-space path length are reported in
Table~\ref{tab:ambient_mbm}. The local metric and trajectory-optimization baselines exhibit clear
scene-dependent behavior. CHOMP performs well in the relatively open
\texttt{box} scene but degrades in the more enclosed \texttt{cage}
scene, while APF and RMP show similar sensitivity to local minima.
RMRRT achieves a $100\%$ success rate in both scenes, indicating that
barrier-informed local geometry combined with sampling-based tree
exploration provides more consistent search behavior across different
scene structures. \textsc{cuRobo} also achieves a $100\%$ success rate and produces the
shortest average paths. Because \textsc{cuRobo} and RMRRT rely on
different computational paradigms---GPU-parallel trajectory
optimization and CPU-based sampling, respectively---their planning
times are reported for completeness but are not interpreted as a
direct hardware-normalized speed comparison.
\begin{figure}[t]
\centering
\includegraphics[width=\columnwidth]{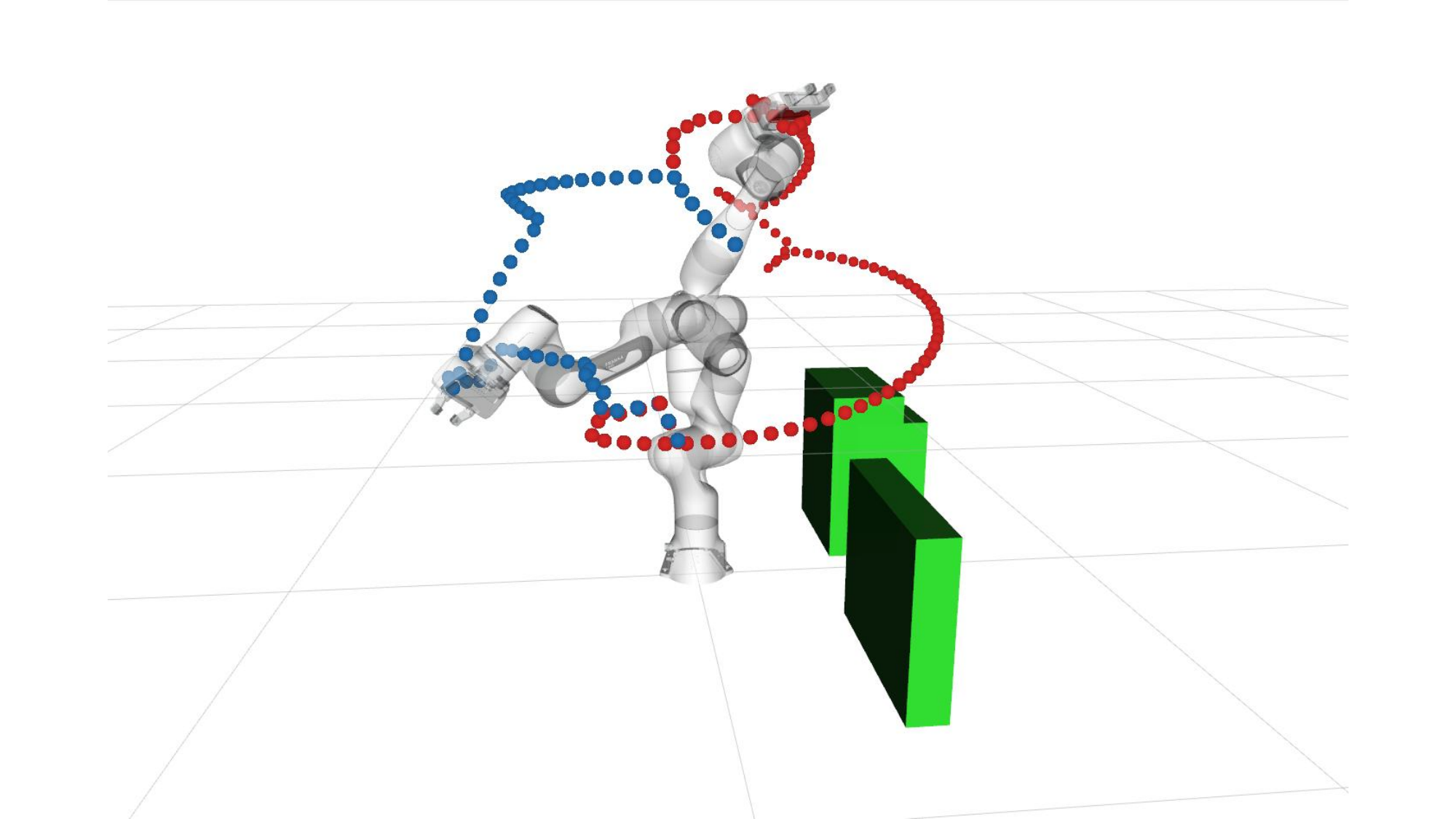}
\caption{
Near-singular reaching paths with and without the manipulability-aware metric. Without the metric, RMRRT approaches a poorly conditioned, nearly extended posture; with the metric, it follows a better-conditioned path. Red and blue spheres denote the corresponding end-effector trajectories.}
\label{fig:appendixC_postures}
\end{figure}

\begin{figure}[t]
\centering
\includegraphics[width=\columnwidth]{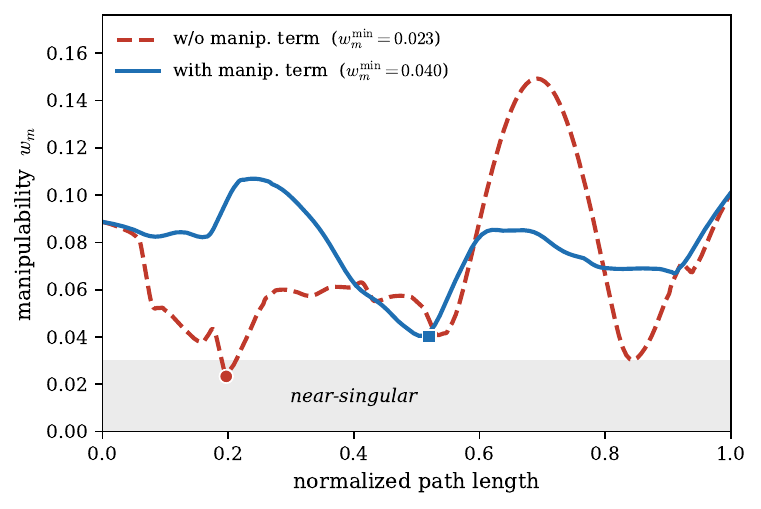}
\caption{
Translational manipulability along the near-singular RMRRT path. Without the manipulability-aware metric (red dashed), the path enters the near-singular region; with the metric (blue solid), the minimum manipulability is increased.}
\label{fig:appendixC_manip}
\end{figure}

\section{Singularity Avoidance via a Manipulability Metric Term}
\label{app:manip}
\noindent This appendix provides a qualitative demonstration of the metric
extensibility discussed in Remark~\ref{remark2}. We augment the ambient
metric with the manipulability-margin contribution
$G_{\mathrm{man}}(q)$ defined therein, while leaving all other components
of RMRRT unchanged. The experiment considers a deliberately constructed
near-singular reaching task for a 7-DoF Franka Panda. Along each planned
path, we evaluate the translational manipulability measure
\[
w_m(q)=\sqrt{\det\!\left(J_p(q)J_p(q)^\top\right)},
\]
where $J_p(q)$ denotes the end-effector position Jacobian. Let
$A(q)=J_r(q)J_r(q)^\top+\epsilon I$. By Jacobi's formula, the
manipulability-margin gradient has components
\[
\frac{\partial h_{\mathrm{man}}}{\partial q_k}
=
2\operatorname{tr}\!\left(
A^{-1}J_r
\left(\frac{\partial J_r}{\partial q_k}\right)^\top
\right).
\]
For a serial chain, the spatial Jacobian columns $\xi_i$ satisfy
$\partial\xi_i/\partial q_k=[\xi_k,\xi_i]$ for $k<i$ and zero otherwise.
In this experiment, $J_r=J_p$. At configurations with $w_m>0$,
$\nabla_q h_{\mathrm{man}}\to 2\nabla_q w_m/w_m$ as
$\epsilon\to0^+$. In the earlier RMRRT implementation used for this
experiment, we estimated $\nabla_q w_m$ by forward finite differences
with a step of $10^{-4}$\,rad and normalized the estimate for use in the
metric term. As shown in Fig.~\ref{fig:appendixC_postures}, the baseline
RMRRT path approaches a nearly fully extended and poorly conditioned
posture, whereas incorporating $G_{\mathrm{man}}(q)$ produces a more
flexed and better-conditioned motion. The corresponding profiles in
Fig.~\ref{fig:appendixC_manip} show that the minimum manipulability
increases from approximately $0.023$ to $0.040$. Near low-margin
configurations, the additional term increases the metric cost of motion
along the manipulability-gradient direction, thereby biasing tree
expansion away from the near-singular region and toward better-conditioned
alternatives. These results provide a qualitative proof of concept that
singularity-related information can be incorporated into RMRRT through
an additional local Riemannian metric term. This experiment does not
establish singularity traversal or restore the regularity assumptions
required by the tangent-space projection and completeness-compatibility
analysis.

\begin{table}[t]

\centering
\caption{Effect of candidate size $K$ on metric-aware nearest-neighbor selection in Scenario~4.}
\label{tab:k_ablation}
\footnotesize
\setlength{\tabcolsep}{4pt}
\begin{tabular}{lcccc}
\toprule
\textbf{Candidate size} & \textbf{Time (s)} & \textbf{Succ.\ (\%)} & \textbf{Path} & \textbf{NN (ms/it)} \\
\midrule
$K{=}1$ (Eucl.\ NN)   & $0.33\pm0.63$ & $100$ & $11.47\pm4.19$ & $0.016$ \\
Adaptive $K$          & $\mathbf{0.26\pm0.36}$ & $100$ & $11.62\pm5.37$ & $0.037$ \\
Large fixed $K$       & $6.11\pm49.9$ & $100$ & $11.46\pm4.02$ & $0.163$ \\
\bottomrule
\end{tabular}
\end{table}
\section{Candidate-Size Analysis for Metric-Aware Nearest-Neighbor Selection}
\label{app:kablation}
\noindent This appendix evaluates the adaptive candidate-size rule in
\eqref{eq:K_adaptive}. We vary only the candidate-pool size $K$ on
Scenario~4, the most challenging 21-DoF triple-arm task, while keeping
all other components of RMRRT fixed. We compare three regimes:
$K=1$, which disables metric-based reranking; the proposed adaptive
rule; and a large fixed $K$ that approximates exhaustive metric search
over the tree. As shown in Table~\ref{tab:k_ablation}, all three regimes solve every
instance, but differ substantially in planning time and runtime
variability. With $K=1$, parent selection depends only on Euclidean
proximity over the tangent-projected displacement, so a nearby node may
still be poorly connected under the barrier-induced metric. A large
fixed $K$ increases the likelihood of selecting a metric-favorable
parent, but its reranking cost grows with the tree and yields diminishing
returns beyond a moderate candidate-set size. The adaptive rule balances these effects by using a compact candidate
set when the tree is small and increasing $K$ logarithmically as the
tree grows. It therefore retains the benefit of metric-aware parent
selection without the computational burden of exhaustive reranking,
yielding the lowest and most consistent planning time.

\bibliographystyle{IEEEtran}
\bibliography{IEEEabrv, references}

@article{mu2025efficient,
  title={Efficient Exploration in Large State-Action Space Through Structured Action Space for Learning Multirobots Motion Planning},
  author={Mu, Chaoxu and Jiang, Zewu and Zhang, Junnan and Wang, Ke and Xu, Xin and Yi, Jun},
  journal={IEEE Transactions on Control Systems Technology},
  year={2025},
  publisher={IEEE}
}

@inproceedings{csucan2012motion,
  title={Motion planning with constraints using configuration space approximations},
  author={{\c{S}}ucan, Ioan A and Chitta, Sachin},
  booktitle={2012 IEEE/RSJ International Conference on Intelligent Robots and Systems},
  pages={1904--1910},
  year={2012},
  organization={IEEE}
}

@inproceedings{sundaralingam2023curobo,
  title={cu{R}obo: Parallelized collision-free robot motion generation},
  author={Sundaralingam, Balakumar and Hari, Siva Kumar Sastry and Fishman, Adam and Garrett, Caelan and Van Wyk, Karl and Blukis, Valts and Millane, Alexander and Oleynikova, Helen and Handa, Ankur and Ramos, Fabio and Fox, Dieter and Boots, Byron},
  booktitle={2023 IEEE International Conference on Robotics and Automation (ICRA)},
  pages={8112--8119},
  year={2023},
  organization={IEEE}
}

@article{chamzas2022motionbenchmaker,
  title={{MotionBenchMaker}: A tool to generate and benchmark motion planning datasets},
  author={Chamzas, Constantinos and Quintero-Pe{\~n}a, Carlos and Kingston, Zachary and Orthey, Andreas and Rakita, Daniel and Gleicher, Michael and Toussaint, Marc and Kavraki, Lydia E},
  journal={IEEE Robotics and Automation Letters},
  volume={7},
  number={2},
  pages={882--889},
  year={2022},
  publisher={IEEE}
}

@inproceedings{zhang2025local,
  title={Local Path Optimization in The Latent Space Using Learned Distance Gradient},
  author={Zhang, Jiawei and Bai, Chengchao and Pan, Wei and Liu, Tianhang and Guo, Jifeng},
  booktitle={2025 IEEE/RSJ International Conference on Intelligent Robots and Systems (IROS)},
  pages={15940--15946},
  year={2025},
  organization={IEEE}
}

@article{cheng2021rmpflow,
  title={Rmpflow: A geometric framework for generation of multitask motion policies},
  author={Cheng, Ching-An and Mukadam, Mustafa and Issac, Jan and Birchfield, Stan and Fox, Dieter and Boots, Byron and Ratliff, Nathan},
  journal={IEEE Transactions on Automation Science and Engineering},
  volume={18},
  number={3},
  pages={968--987},
  year={2021},
  publisher={IEEE}
}

@inproceedings{stilman2007task,
  title={Task constrained motion planning in robot joint space},
  author={Stilman, Mike},
  booktitle={2007 IEEE/RSJ International Conference on Intelligent Robots and Systems},
  pages={3074--3081},
  year={2007},
  organization={IEEE}
}

@article{xian2017closed,
  title={Closed-chain manipulation of large objects by multi-arm robotic systems},
  author={Xian, Zhou and Lertkultanon, Puttichai and Pham, Quang-Cuong},
  journal={IEEE Robotics and Automation Letters},
  volume={2},
  number={4},
  pages={1832--1839},
  year={2017},
  publisher={IEEE}
}

@inproceedings{berenson2009manipulation,
  title={Manipulation planning on constraint manifolds},
  author={Berenson, Dmitry and Srinivasa, Siddhartha S and Ferguson, Dave and Kuffner, James J},
  booktitle={2009 IEEE international conference on robotics and automation},
  pages={625--632},
  year={2009},
  organization={IEEE}
}

@article{kingston2018sampling,
  title={Sampling-based methods for motion planning with constraints},
  author={Kingston, Zachary and Moll, Mark and Kavraki, Lydia E},
  journal={Annual review of control, robotics, and autonomous systems},
  volume={1},
  number={1},
  pages={159--185},
  year={2018},
  publisher={Annual Reviews}
}

@article{jaillet2012path,
  title={Path planning under kinematic constraints by rapidly exploring manifolds},
  author={Jaillet, L{\'e}onard and Porta, Josep M},
  journal={IEEE Transactions on Robotics},
  volume={29},
  number={1},
  pages={105--117},
  year={2012},
  publisher={IEEE}
}

@inproceedings{suh2011tangent,
  title={Tangent space RRT: A randomized planning algorithm on constraint manifolds},
  author={Suh, Chansu and Um, Terry Taewoong and Kim, Beobkyoon and Noh, Hakjong and Kim, Munsang and Park, Frank C},
  booktitle={2011 IEEE International Conference on Robotics and Automation},
  pages={4968--4973},
  year={2011},
  organization={IEEE}
}

@inproceedings{jaillet2016path,
  title={Path planning with loop closure constraints using an atlas-based RRT},
  author={Jaillet, L{\'e}onard and Porta, Josep M},
  booktitle={Robotics Research: The 15th International Symposium ISRR},
  pages={345--362},
  year={2016},
  organization={Springer}
}

@article{kim2016tangent,
  title={Tangent bundle RRT: A randomized algorithm for constrained motion planning},
  author={Kim, Beobkyoon and Um, Terry Taewoong and Suh, Chansu and Park, Frank C},
  journal={Robotica},
  volume={34},
  number={1},
  pages={202--225},
  year={2016},
  publisher={Cambridge University Press}
}

@article{quiroz2019whole,
  title={Whole-body control with (self) collision avoidance using vector field inequalities},
  author={Quiroz-Oma{\~n}a, Juan Jos{\'e} and Adorno, Bruno Vilhena},
  journal={IEEE Robotics and Automation Letters},
  volume={4},
  number={4},
  pages={4048--4053},
  year={2019},
  publisher={IEEE}
}

@article{guo2012new,
  title={A new inequality-based obstacle-avoidance MVN scheme and its application to redundant robot manipulators},
  author={Guo, Dongsheng and Zhang, Yunong},
  journal={IEEE Transactions on Systems, Man, and Cybernetics, Part C (Applications and Reviews)},
  volume={42},
  number={6},
  pages={1326--1340},
  year={2012},
  publisher={IEEE}
}

@inproceedings{grassmann2018smooth,
  title={Smooth point-to-point trajectory planning in $ SE $(3) with self-collision and joint constraints avoidance},
  author={Grassmann, Reinhard and Johannsmeier, Lars and Haddadin, Sami},
  booktitle={2018 IEEE/RSJ International Conference on Intelligent Robots and Systems (IROS)},
  pages={1--9},
  year={2018},
  organization={IEEE}
}

@inproceedings{hsu2003bridge,
  title={The bridge test for sampling narrow passages with probabilistic roadmap planners},
  author={Hsu, David and Jiang, Tingting and Reif, John and Sun, Zheng},
  booktitle={2003 IEEE international conference on robotics and automation (cat. no. 03CH37422)},
  volume={3},
  pages={4420--4426},
  year={2003},
  organization={IEEE}
}

@article{pan2016fast,
  title={Fast probabilistic collision checking for sampling-based motion planning using locality-sensitive hashing},
  author={Pan, Jia and Manocha, Dinesh},
  journal={The International Journal of Robotics Research},
  volume={35},
  number={12},
  pages={1477--1496},
  year={2016},
  publisher={Sage Publications Sage UK: London, England}
}

@InProceedings{Hu_2021_CVPR,
    author    = {Hu, Peiyun and Huang, Aaron and Dolan, John and Held, David and Ramanan, Deva},
    title     = {Safe Local Motion Planning With Self-Supervised Freespace Forecasting},
    booktitle = {Proceedings of the IEEE/CVF Conference on Computer Vision and Pattern Recognition (CVPR)},
    month     = {June},
    year      = {2021},
    pages     = {12732-12741}
}

@article{cai2025just,
  title={Just-in-Time Informed Trees: Manipulability-Aware Asymptotically Optimized Motion Planning},
  author={Cai, Kuanqi and Zhang, Liding and Su, Xinwen and Chen, Kejia and Wang, Chaoqun and Haddadin, Sami and Knoll, Alois and Ajoudani, Arash and Figueredo, Luis},
  journal={IEEE/ASME Transactions on Mechatronics},
  year={2025},
  publisher={IEEE}
}

@article{li2024quadratic,
  title={A quadratic programming solution to position-level repetitive motion planning and obstacle avoidance of joint-constrained robots},
  author={Li, Weibing and Yi, Zilian and Rao, Ruiqi and Guo, Dongsheng and Pan, Yongping and Huang, Kai},
  journal={IEEE Transactions On Industrial Electronics},
  year={2024},
  publisher={IEEE}
}

@inproceedings{ratliff2009chomp,
  title={CHOMP: Gradient optimization techniques for efficient motion planning},
  author={Ratliff, Nathan and Zucker, Matt and Bagnell, J Andrew and Srinivasa, Siddhartha},
  booktitle={2009 IEEE international conference on robotics and automation},
  pages={489--494},
  year={2009},
  organization={IEEE}
}

@article{khatib1986real,
  title={Real-time obstacle avoidance for manipulators and mobile robots},
  author={Khatib, Oussama},
  journal={The international journal of robotics research},
  volume={5},
  number={1},
  pages={90--98},
  year={1986},
  publisher={Sage Publications Sage CA: Thousand Oaks, CA}
}

@inproceedings{manjunath2021safe,
  title={Safe and robust motion planning for dynamic robotics via control barrier functions},
  author={Manjunath, Aniketh and Nguyen, Quan},
  booktitle={2021 60th IEEE Conference on Decision and Control (CDC)},
  pages={2122--2128},
  year={2021},
  organization={IEEE}
}

@inproceedings{yang2019sampling,
  title={Sampling-based motion planning via control barrier functions},
  author={Yang, Guang and Vang, Bee and Serlin, Zachary and Belta, Calin and Tron, Roberto},
  booktitle={Proceedings of the 2019 3rd International Conference on Automation, Control and Robots},
  pages={22--29},
  year={2019}
}

@article{bordalba2020randomized,
  title={A randomized kinodynamic planner for closed-chain robotic systems},
  author={Bordalba, Ricard and Ros, Llu{\'\i}s and Porta, Josep M},
  journal={IEEE Transactions on Robotics},
  volume={37},
  number={1},
  pages={99--115},
  year={2020},
  publisher={IEEE}
}

@article{jang2022motion,
  title={Motion planning for closed-chain constraints based on probabilistic roadmap with improved connectivity},
  author={Jang, Keunwoo and Baek, Jiyeong and Park, Suhan and Park, Jaeheung},
  journal={IEEE/ASME Transactions on Mechatronics},
  volume={27},
  number={4},
  pages={2035--2043},
  year={2022},
  publisher={IEEE}
}

@article{park2024constrained,
  title={A constrained motion planning method exploiting learned latent space for high-dimensional state and constraint spaces},
  author={Park, Suhan and Jeon, Suhyun and Park, Jaeheung},
  journal={IEEE/ASME Transactions on Mechatronics},
  volume={29},
  number={4},
  pages={3001--3009},
  year={2024},
  publisher={IEEE}
}

@article{qureshi2022compnetx,
  title={Constrained Motion Planning Networks X},
  author={Qureshi, Ahmed H. and Dong, Jiangeng and Baig, Asfiya and Yip, Michael C.},
  journal={IEEE Transactions on Robotics},
  volume={38},
  number={2},
  pages={868--886},
  year={2022},
  publisher={IEEE}
}

@article{liu2025planning,
  title={A planning framework for complex flipping manipulation of multiple mobile manipulators},
  author={Liu, Wenhang and Ren, Meng and Song, Kun and Wang, Michael Yu and Xiong, Zhenhua},
  journal={IEEE Robotics and Automation Letters},
  year={2025},
  publisher={IEEE}
}

@article{dogar2019multi,
  title={Multi-robot grasp planning for sequential assembly operations},
  author={Dogar, Mehmet and Spielberg, Andrew and Baker, Stuart and Rus, Daniela},
  journal={Autonomous Robots},
  volume={43},
  number={3},
  pages={649--664},
  year={2019},
  publisher={Springer}
}

@article{zimmermann2020multi,
  title={A multi-level optimization framework for simultaneous grasping and motion planning},
  author={Zimmermann, Simon and Hakimifard, Ghazal and Zamora, Miguel and Poranne, Roi and Coros, Stelian},
  journal={IEEE Robotics and Automation Letters},
  volume={5},
  number={2},
  pages={2966--2972},
  year={2020},
  publisher={IEEE}
}

@article{shi2025distributed,
  title={Distributed quadratic programming for safe human multirobot interaction under task constraints},
  author={Shi, Kaige and Liu, Jinxin and Chang, Jindong and Hu, Guoqiang},
  journal={IEEE/ASME Transactions on Mechatronics},
  year={2025},
  publisher={IEEE}
}

@inproceedings{alonso2015local,
  title={Local motion planning for collaborative multi-robot manipulation of deformable objects},
  author={Alonso-Mora, Javier and Knepper, Ross and Siegwart, Roland and Rus, Daniela},
  booktitle={2015 IEEE international conference on robotics and automation (ICRA)},
  pages={5495--5502},
  year={2015},
  organization={IEEE}
}

@article{kavraki1996probabilistic,
  title={Probabilistic roadmaps for path planning in high-dimensional configuration spaces},
  author={Kavraki, Lydia E and {\v{S}}vestka, Petr and Latombe, Jean-Claude and Overmars, Mark H},
  journal={IEEE Transactions on Robotics and Automation},
  volume={12},
  number={4},
  pages={566--580},
  year={1996},
  publisher={IEEE}
}

@article{lavalle2001randomized,
  title={Randomized kinodynamic planning},
  author={LaValle, Steven M and Kuffner, James J},
  journal={The International Journal of Robotics Research},
  volume={20},
  number={5},
  pages={378--400},
  year={2001},
  publisher={SAGE Publications}
}

@inproceedings{hsu1997expansive,
  title={Path planning in expansive configuration spaces},
  author={Hsu, David and Latombe, Jean-Claude and Motwani, Rajeev},
  booktitle={Proceedings of International Conference on Robotics and Automation},
  volume={3},
  pages={2719--2726},
  year={1997},
  organization={IEEE}
}

@article{kingston2019exploring,
  title={Exploring implicit spaces for constrained sampling-based planning},
  author={Kingston, Zachary and Moll, Mark and Kavraki, Lydia E},
  journal={The International Journal of Robotics Research},
  volume={38},
  number={10-11},
  pages={1151--1178},
  year={2019},
  publisher={SAGE Publications}
}

@inproceedings{bonalli2019trajectory,
  title={Trajectory optimization on manifolds: A theoretically-guaranteed embedded sequential convex programming approach},
  author={Bonalli, Riccardo and Bylard, Andrew and Cauligi, Abhishek and Lew, Thomas and Pavone, Marco},
  booktitle={Robotics: Science and Systems (RSS)},
  year={2019}
}

@article{bordalba2022direct,
  title={Direct collocation methods for trajectory optimization in constrained robotic systems},
  author={Bordalba, Ricard and Schoels, Tobias and Ros, Llu{\'\i}s and Porta, Josep M and Diehl, Moritz},
  journal={IEEE Transactions on Robotics},
  volume={39},
  number={1},
  pages={183--202},
  year={2022},
  publisher={IEEE}
}

@inproceedings{cohn2024constrained,
  title={Constrained bimanual planning with analytic inverse kinematics},
  author={Cohn, Thomas and Shaw, Seiji and Simchowitz, Max and Tedrake, Russ},
  booktitle={2024 IEEE International Conference on Robotics and Automation (ICRA)},
  pages={6935--6942},
  year={2024},
  organization={IEEE}
}

@article{bohigas2013planning,
  title={Planning singularity-free paths on closed-chain manipulators},
  author={Bohigas, Oriol and Henderson, Michael E and Ros, Llu{\'\i}s and Manubens, Montserrat and Porta, Josep M},
  journal={IEEE Transactions on Robotics},
  volume={29},
  number={4},
  pages={888--898},
  year={2013},
  publisher={IEEE}
}

@article{ratliff2018riemannian,
  title={Riemannian motion policies},
  author={Ratliff, Nathan D and Issac, Jan and Kappler, Daniel and Birchfield, Stan and Fox, Dieter},
  journal={arXiv preprint arXiv:1801.02854},
  year={2018}
}

@book{boyd2004convex,
  title={Convex Optimization},
  author={Boyd, Stephen and Vandenberghe, Lieven},
  year={2004},
  publisher={Cambridge University Press}
}

@book{nocedal2006numerical,
  title={Numerical Optimization},
  author={Nocedal, Jorge and Wright, Stephen J},
  edition={2nd},
  year={2006},
  publisher={Springer}
}

@book{golub2013matrix,
  title={Matrix Computations},
  author={Golub, Gene H and Van Loan, Charles F},
  edition={4th},
  year={2013},
  publisher={Johns Hopkins University Press}
}

\newpage
\begin{IEEEbiography}[{\includegraphics[width=1in,height=1.25in,clip,keepaspectratio]{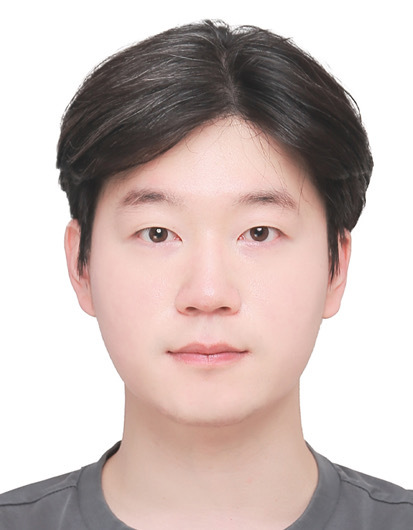}}]{Minhyeong Kang} (Graduate Student Member, IEEE) received the B.S. degree in Mechanical Engineering from Kyung Hee University, South Korea, in 2025, where he is currently pursuing the Ph.D. degree in Mechanical Engineering. His research interests include optimal control and motion planning of redundant robots, including dual-arm manipulators and humanoids.
\end{IEEEbiography}
\begin{IEEEbiography}[{\includegraphics[width=1in,height=1.25in,clip,keepaspectratio]{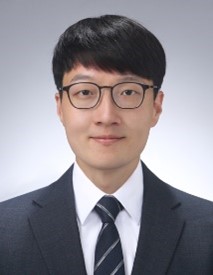}}]{Sanghyun Kim}
(Member, IEEE) received the B.S. degree in mechanical engineering and the Ph.D. degree in transdisciplinary studies from Seoul National University, Seoul, South Korea, in 2012 and 2020, respectively.
In 2020, Kim was a Postdoctoral Fellow with the Statistical Machine Learning and Motor Control Group, University of Edinburgh, Edinburgh, U.K.
From 2020 to 2023, Kim was a Senior Researcher with the AI Robot Research Division, Korea Institute of Machinery and Materials (KIMM), Daejeon, South Korea.
Since 2023, Kim has been an Assistant Professor with the Department of Mechanical Engineering, Kyung Hee University, Yongin, South Korea, where Kim directs the Robot Control and Intelligence Laboratory.
Research interests include humanoid robotics, whole-body control, task and motion planning, optimal control, and robot learning.
\end{IEEEbiography}
\vfill

\end{document}